\documentclass{article}

\usepackage{microtype}
\usepackage{graphicx}
\usepackage{subfigure}
\usepackage{booktabs} 
\usepackage{multirow}
\usepackage{hyperref}

\usepackage[accepted]{icml2023}

\usepackage{amsmath}
\usepackage{amssymb}
\usepackage{mathtools}
\usepackage{amsthm}

\usepackage[capitalize,noabbrev]{cleveref}

\theoremstyle{plain}
\newtheorem{theorem}{Theorem}[section]
\newtheorem*{theorem*}{Theorem}

\theoremstyle{definition}
\newtheorem{definition}[theorem]{Definition}

\theoremstyle{remark}
\newtheorem{remark}[theorem]{Remark}

\usepackage[textsize=tiny]{todonotes}

\def\x{{\boldsymbol{x}}}
\def\bo{{\boldsymbol{\omega}}}
\def\bb{{\boldsymbol{\beta}}}

\def\g{{\boldsymbol{g}}}
\def\a{{\boldsymbol{a}}}
\def\c{{\boldsymbol{c}}}

\def\b{{\boldsymbol{b}}}
\def\0{{\boldsymbol{0}}}
\def\h{{\mathbf{h}}}
\def\trans{^{\rm T}}
\def\f{{\mathbf{f}}}

\icmltitlerunning{Efficient Label Shift Adaptation}

\begin{document}

\twocolumn[
\icmltitle{ELSA: Efficient Label Shift Adaptation through the Lens of Semiparametric Models}



\icmlsetsymbol{equal}{*}

\begin{icmlauthorlist}
\icmlauthor{Qinglong Tian}{equal,waterloo}
\icmlauthor{Xin Zhang}{equal,meta}
\icmlauthor{Jiwei Zhao}{madison}
\end{icmlauthorlist}

\icmlaffiliation{waterloo}{Department of Statistics and Actuarial Science, University of Waterloo, Waterloo, ON, Canada}
\icmlaffiliation{meta}{Meta Inc., Menlo Park, CA, USA}
\icmlaffiliation{madison}{Department of Biostatistics and Medical Informatics, University of Wisconsin-Madison, Madison, WI, USA}

\icmlcorrespondingauthor{Jiwei Zhao}{jiwei.zhao@wisc.edu}

\icmlkeywords{Label Shift, Domain Adaptation, Moment Matching, Semiparametric Model}

\vskip 0.3in
]



\printAffiliationsAndNotice{\icmlEqualContribution} 

\begin{abstract}

We study the domain adaptation problem with label shift in this work.
Under the label shift context, the marginal distribution of the label varies across the training and testing datasets, while the conditional distribution of features given the label is the same.
Traditional label shift adaptation methods either suffer from large estimation errors or require cumbersome post-prediction calibrations.
To address these issues, we first propose a moment-matching framework for adapting the label shift based on the geometry of the influence function.
Under such a framework, we propose a novel method named \underline{E}fficient \underline{L}abel \underline{S}hift \underline{A}daptation (ELSA), in which the adaptation weights can be estimated by solving linear systems.
Theoretically, the ELSA estimator is $\sqrt{n}$-consistent ($n$ is the sample size of the source data) and asymptotically normal.
Empirically, we show that ELSA can achieve state-of-the-art estimation performances without post-prediction calibrations, thus, gaining computational efficiency.

\end{abstract}


\section{Introduction}

\subsection{Background}

Traditional supervised learning assumes that the training and testing data are from the same joint distribution $p(\x,y)$, where $\x$ are the features, and $y$ is the label.
Therefore, we can apply the predictive model from the training data to do inferences on the testing data.
However, in many real-world applications, the joint distribution $p(\x,y)$ may often differ in the training and testing data, and this phenomenon is called \textit{distributional shift} (see \citealt{quinonero2008dataset}).
Thus, the knowledge learned from the training data may no longer be appropriate to make predictions directly from the testing data.
This motivates the study direction called \textit{unsupervised domain adaptation} (see \citealt{kouw2021}), which aims to address the distributional shift between the source domain $p_s(\x,y)$ (i.e., training data) and the target domain $p_t(\x,y)$ (i.e., testing data), where the subscript $s$ and $t$ represent source and target domain respectively.

In this work, we focus on studying the case of \textit{label shift}, also known as prior probability shift (see \citealt{moreno2012}).
The label shift refers to the phenomenon that the marginal distributions of the labels differ in the source and target domains $p_s(y)\neq p_t(y)$ though the conditional distributions of features given the label are the same in both domains $p_s(\x| y)=p_t(\x| y)$.
Label shift aligns with the anticausal learning setting (see \citealt{storkey2009}) where the label $y$ is the cause and the features $\x$ are the effects.
For example, in the task of disease diagnostic, suppose the label $y$ indicates whether a person has been infected, and $\x$ are the observed symptoms.
It is reasonable to assume the identical conditional distributions $p(\x|y)$, as the infection-on-symptoms mechanism should be the same.
Consider the source and target data are from two regions with and without the corresponding prevention so that $p_s(y)\neq p_t(y)$.
Then with the same symptoms, people without prevention are more likely to be infected by the disease compared to those with good prevention (i.e., $p_s(y|\x)\neq p_t(y|\x)$).
Thus, it is important to perform  label shift adaptation to make the trained predictive model applicable for the inference on the target domain.


\subsection{Problem Settings and Preparations}

We first formally define the problem and notations.
The observable data consists of two parts: the labeled source data $\left\{(\x_i,y_i)\right\}_{i=1}^{n}$ and unlabeled target data $\left\{\x_i\right\}_{i=n+1}^{n+m}$.
A generic notation $p_s(\cdot)$ and $p_t(\cdot)$ denotes distributions on the source and target domain, respectively.
The label $y$ is a $k$-class variable with support $\mathcal{Y}=\left\{1,\dots,k\right\}$ on the source domain.
The support on the target domain is a subset of $\mathcal{Y}$, which implies the testing data do not contain new classes that are not in the training data.
The features are denoted by $\x\in\mathbb{R}^{p}$, and $\x$ could be high-dimensional vector.

Under the label shift context, $p_s(y|\x)$ and $p_t(y|\x)$ are generally different.
Thus, a model trained using labeled source data (i.e., $\widehat{p}_s(y|\x)$) is usually biased.
Using such a model for $p_t(y|\x)$ will decrease prediction accuracy.
If we have prior knowledge about $p_s(y)$ and $p_t(y)$, we can use the following Bayes formula to adjust the trained model for testing data:
\begin{align}\label{eq:bayes-formula}
&p_t(y=i|\x)=\frac{p_t(y=i,\x)}{p_t(\x)}=\frac{p_t(y=i,\x)}{\sum_{j=1}^{k}p_t(\x,y=j)}\notag\\
&=\frac{\dfrac{p_t(y=i)}{p_s(y=i)}p_s(y=i|\x)}{\sum_{j=1}^{k}\dfrac{p_t(y=j)}{p_s(y=j)}p_s(y=j|\x)},\quad i=1,\dots,k.
\end{align}
Equation~(\ref{eq:bayes-formula}) provides the insight on turning $[p_s(y=1|\x),\dots, p_s(y=k|\x)]$ into $[p_t(y=1|\x),\dots, p_t(y=k|\x)]$.
The formula hinges on the knowledge of the importance weights $\bo=(\omega_1,\dots,\omega_k)$, where $\omega_i\coloneqq p_t(y=i)/p_s(y=i)$.
Although we can estimate $p_s(y=i)$ $(i=1,\dots,k)$ empirically, data from the target distribution are unlabeled; thus, we cannot simply take the empirical distribution of $p_t(y)$ to estimate $\bo$ with the ratio $\widehat{p}_t(y=i)/\widehat{p}_s(y=i)$.
Hence, we can convert addressing the label shift into the estimation problem of $\bo$.

\subsection{Related Work and Existing Methods}

In the literature, the label shift adaptation has attracted increasing attention.
The well-known methods include the kernel-mean-matching approach \cite{zhang2013domain}, the generative adversarial training approach \cite{pmlr-v119-guo20d} and importance re-weighting approaches \cite{maity2020minimax, evans2021inference, roberts2022unsupervised}.
In this work, we follow the research track of using importance re-weighting to address the label shift problem. 
In the literature, there are two popular types of importance weight estimation methods:
One is based on inverting a confusion matrix, and the other is maximizing the likelihood function. In the following, we will briefly review the two methods.

The confusion matrix method is based on solving a linear equation.
Under the label shift assumption, it holds 
\begin{equation}\label{eq:bbse-eq}
\int\omega(y)p_s(\x,y)dy=p_t(\x),
\end{equation}
where $\omega(y)=p_t(y)/p_s(y)$.
The black box shift estimation (BBSE) method (see \citealt{lipton2018detecting}) replaces $\x$ with $\widehat{y}=f(\x)\in\left\{1,\dots,k\right\}$ in (\ref{eq:bbse-eq}), where $f(\x)=\mathrm{argmax}_{y}\widehat{p}_s(y|\x)$ and $\widehat{p}_s(y|\x)$ is a trained model.
Then (\ref{eq:bbse-eq}) becomes a linear system where we can solve the importance weights by inverting a confusion matrix.
\citet{azizzadenesheli2018regularized} proposed the regularized learning of label shift (RLLS) method.
The RLLS method added a regularized scheme to the BBSE method by penalizing on $\Vert\bo-\mathbf{1}\Vert$, but the essence of the RLLS method is the same as the BBSE, which is inverting a confusion matrix.

The other type of method is based on maximizing the likelihood function.
\citet{saerens2002adjusting} proposed the maximum likelihood label shift (MLLS) method by maximizing $\sum_{i=n+1}^{n+m}\log p_t(\x_t)$ with respect to $\bo$.
An EM algorithm was proposed to perform the maximum likelihood estimation.
However, the MLLS method needs the model $p_s(y|\x)$ to be correctly specified to ensure that the estimation is consistent.
\citet{alexandari2020} showed that the MLLS method performed poorly when $p_s(y|\x)$ was fitted using a vanilla neural network model.
\citet{alexandari2020} further discovered that one could greatly improve the estimation of the importance weights by calibrating the fitted predictive model $\widehat{p}_s(y|\x)$, but it remained unclear why calibration can help.
\citet{garg2020} provided theoretical justifications for calibrating the predictive model by showing maximizing likelihood is equivalent to minimizing a KL divergence if the predictive model is calibrated.
The calibrated MLLS method outperforms the confusion matrix methods (i.e., BBSE and RLLS) in estimation error.

\subsection{Contributions and Outlines}

The confusion matrix methods are easy to implement, but the performances are less satisfactory than the calibrated MLLS method.
But the calibration procedure is non-trivial, requiring further training and optimizations (e.g., \citealt{guo2017calibration}).
Furthermore, there is more than one calibration method, and the final results are sensitive to the choice of the method.

This paper aims to find a solution combining the advantages of both methods: achieving good performance while keeping the procedure simple.
We achieve this goal by proposing a moment-matching estimator.
The major contributions of our work are summarized as follows:
\begin{enumerate}
    \item We derive a novel moment-matching framework to address the label shift problem.
    The proposed moment-matching framework is based on the geometry of the influence function under a semiparametric model.
    Under the proposed moment-matching framework, we develop the estimator for the importance weight $\bo$.
    It has been shown that our proposed estimators are regular asymptotic linear (RAL) (see \citealt{tsiatis2006semiparametric}).
    \item We propose the efficient label shift adaptation (ELSA) method for the importance weight estimation.
    We can obtain the ELSA estimator by solving linear systems.
    We prove that the ELSA estimator is $\sqrt{n}$-consistent,
    We further show that the ELSA estimator satisfies the asymptotic normality. 
    \item  We conduct thorough numerical experiments to validate the performance of our proposed ELSA method. 
    The ELSA method outperforms the well-known methods: BBSE, RLLS, and MLLS.
    Furthermore, the ELSA method has competitive performance and is more computationally efficient than the calibrated MLLS method, which is the best method in the literature (see \citealt{alexandari2020}).
\end{enumerate}

The rest of the paper is organized as follows. In Section~\ref{sec:formulation}, we formulate the label shift problem semiparametrically and propose a moment-matching framework for estimating the importance weights.
Section~\ref{sec:proposed-estimator} proposes the ELSA estimator based on the aforementioned moment-matching framework; we further prove that the ELSA estimator is consistent and has an asymptotic normal distribution.
Section~\ref{sec:numerical-studies} uses numerical studies to compare the ELSA method with other existing methods.
Section~\ref{sec:discussion} concludes the paper with some discussions on future work.


\section{Semiparametric Moment-Matching Framework}
\label{sec:formulation}

This section proposes a novel moment-matching framework for addressing the label shift problem.
The framework is based on the geometry of influence functions from a semiparametric model.
More specifically, we need to derive the function space perpendicular to the nuisance tangent space.
We refer readers to Section~\ref{sup:pre} of the appendix for prerequisites and more details on the RAL estimator, influence function, and geometry of function spaces.

In the first subsection, we unify the source and target distributions using one imaginary pooled distribution and describe the data with a semiparametric likelihood function.
In the second subsection, we derive the perpendicular space under the semiparametric model.
We then propose a moment-matching framework for addressing the label shift from the perpendicular space.


\subsection{A Semiparametric Model}

We can write the importance weight as a function of $y$ as
\[
\omega(y;\bo)\coloneqq
\frac{p_t(y)}{p_s(y)}=\sum_{i=1}^{k}\omega_i\mathrm{I}(y=i),\quad y=1,\dots,k,
\]
where $\bo=(\omega_1,\dots,\omega_k)$.
In the rest of the paper, we write the importance weight function as $\omega(y)$ for short.
Under the unsupervised domain adaptation setting, the labeled data from the source domain are independent of the unlabeled target data.
The sample sizes of both data can be arbitrary and unrelated, and the two distributions are not identically distributed.
We introduce a binary variable $r$ so that the source and target distributions merge into one fictitious pooled distribution.
We have a random triplet $(\x,y,r)$ in this pooled distribution.
A sample with $r=1$ belongs to the source  domain and $p_s(\x,y)=p(\x,y|r=1)$ and $r=0$ verse visa.
Using this pooled distribution $p(\x,y,r)$, we can treat the pooled data as independent and identically distributed.

For any sample from the pooled distribution, its likelihood function is given by
\begin{equation}\label{eq:original-likelihood}
\mathcal{L}=\left\{\pi p_s(\x,y)\right\}^{r}\left\{(1-\pi)p_t(\x)\right\}^{1-r}\\
\end{equation}
where $\pi=\Pr(r=1)$.
We can decompose the likelihood in (\ref{eq:original-likelihood}) as
\begin{equation}\label{eq:likelihood-1}
\begin{aligned}
\mathcal{L}\!=\!\pi^r(1\!-\!\pi)^{1\!-\!r}p_s(\x)\{p_s(y|\x)\}^r\Bigg\{\!\!\!\!\int\omega(y)p_s(y|\x)dy\Bigg\}^{\!1\!-\!r}
\end{aligned}
\end{equation}
In (\ref{eq:likelihood-1}), $\bo$ are of interest while $p_s(y|\x)$, $p_s(\x)$, and $\pi$ are nuisances.
From a parametric perspective, one can estimate $\bo$ using the maximum likelihood estimator (MLE) if $p_s(y|\x)$ is correct.
As $p_s(\x)$ and $\pi$ are irrelevant with the MLE of $\bo$, if we were able to correctly specify $p_s(y|\x)$, the likelihood function could be reduced to $\mathcal{L}\propto\left\{\int\omega(y)p_s(y|\x)dy\right\}^{1-r}$.
This is exactly the method proposed in \citet{saerens2002adjusting}.
However, The success of this method hinges on whether one can correctly specify $p_s(y|\x)$.
If the $p_s(y|\x)$ is misspecified, the MLE of $\bo$ may not be consistent anymore.


To avoid such an issue, we can view the likelihood (\ref{eq:original-likelihood}) from a semiparametric perspective, where $p_s(y|\x)$ is no longer required to be a correctly specified parametric model.
We reformulate the likelihood (\ref{eq:original-likelihood}) as
\begin{align}\label{eq:semi-likelihood}
    \mathcal{L}=\pi^{r}(1-\pi)^{1-r}&\{p_s(y)\}^r
    \{p(\x| y)\}^{r} \notag\\
   &\cdot\Bigg\{\int\omega(y)p(\x| y)p_s(y)dy\Bigg\}^{1-r}\!\!\!\!\!\!.
\end{align}
We treat $p(\x|y)$, $p_s(y)$, and $\pi$ in (\ref{eq:semi-likelihood}) as nuisances.
Note that we do {\em not} assume any parametric models for $p(\x|y)$ and $p_s(y)$, and thus they are allowed to be nonparametric functions from the infinite-dimensional space.
Although we rewrite the likelihood function using $p(\boldsymbol{x}|y)$ in (5), we do not intend to estimate $p(\boldsymbol{x}|y)$. Here $p(\boldsymbol{x}|y)$ is solely an intermediate component for deriving the perpendicular space (as detailed in Section~\ref{sup:pre} of the appendix).
The reason for choosing $p(\boldsymbol{x}|y)$ is that $p(\boldsymbol{x}|y)$ is the same in both the source and target distributions so that we do not need to differentiate $p_s(y|\boldsymbol{x})$ and $p_t(y|\boldsymbol{x})$.

It is infeasible to compute the MLE of $\bo$ in (\ref{eq:semi-likelihood}) when we assume a nonparametric (i.g., infinite-dimensional) model for $p(\x|y)$.
Instead, our solution is to geometrically blend the parameter of interest with other nuisances so that we can develop {\em robust} estimation methods even if the model $p(\x|y)$ is misspecified with respect to nuisances.
To be more specific, we first compute the nuisance tangent spaces and find the function space perpendicular to all the nuisance tangent spaces.
The influence functions of estimators lay in the perpendicular space.
In addition to Section~\ref{sup:pre} of the appendix, we refer to \citet{bickel1993efficient} and \citet{tsiatis2006semiparametric} for more details.




Based on the above semiparametric likelihood function in (\ref{eq:semi-likelihood}), we will derive the nuisance tangent space and the perpendicular space in Section~\ref{Sec. function_spaces_and_matching}, which is essential for developing our label shift adaptation estimator.

\subsection{Function Spaces and Moment Matching}\label{Sec. function_spaces_and_matching}

Our main utilization of semiparametric models is to derive the complement of the nuisance tangent space (i.e., the perpendicular space) $\Lambda^{\perp}$. Based on the semiparametric theory \citep{bickel1993efficient,tsiatis2006semiparametric}, this space corresponds to the influence functions for estimating the parameter of interest $\boldsymbol{\omega}$. In other words, every element in $\Lambda^{\perp}$ corresponds to a RAL estimator of $\boldsymbol{\omega}$. Also, this space indicates any function that is not in this space should not be used for estimating $\boldsymbol{\omega}$ in the interest of efficiency. For example, if $\boldsymbol{\phi}$ is an function that $\boldsymbol{\phi}\not\in\Lambda^{\perp}$, then one should not use $\boldsymbol{\phi}$ but to use $\Pi(\boldsymbol{\phi}|\Lambda^{\perp})$ instead. Here
$$
\boldsymbol{\phi}=\underbrace{\boldsymbol{\phi}-\Pi(\boldsymbol{\phi}|\Lambda^{\perp})}_{\in\Lambda}\oplus\underbrace{\Pi(\boldsymbol{\phi}|\Lambda^{\perp})}_{\in\Lambda^{\perp}},
$$
and $\Pi(\boldsymbol{\phi}|\Lambda^{\perp})$ is the projection of $\boldsymbol{\phi}$ onto the space $\Lambda^{\perp}$. This is because by using the projection $\Pi(\boldsymbol{\phi}|\Lambda^{\perp})$, we can improve the efficiency (i.e., empirically, decrease the MSE of the estimator). Also, if $\boldsymbol{\phi}\not\in\Lambda^{\perp}$, we cannot obtain a RAL estimator, thus, making it difficult to characterize the resulting estimator.

Our ultimate goal is to find a RAL estimator for $\bo$, which would enjoy the consistency and asymptotic normality properties (see Section~\ref{sup:pre} of the appendix).
Achieving this goal is equivalent to finding its corresponding influence function.
Because an influence function lays in a function space perpendicular to the nuisance tangent spaces, our goal becomes to derive the perpendicular space, which is the main goal of Theorem~\ref{th:perpendicualr}.
\begin{theorem}\label{th:perpendicualr}
Without a loss of generality, we fix $w_k$ as $w_k=(1-\sum_{i=1}^{k-1}w_ip_i)/p_k$, and $p_k$ as $p_k=1-\sum_{i=1}^{k-1}p_i$, where $p_i\coloneqq \Pr{}_{\!s}(y=i)$.
The log-likelihood function is given by $r\log(\pi)+(1-r)\log(1-\pi)+\sum_{i=1}^{k-1}r\mathrm{I}(y=i)\left\{\log p(\x| y=i)+\log(p_i)\right\}+r\mathrm{I}(y=k)\left\{\log p(\x| y=k)+\log(1-p_1-\dots-p_{k-1})\right\}+(1-r)\log\{p(\x| y=1) w_1p_1+p(\x| y=2)w_2p_2+\dots+p(\x| y=k)(1-w_1p_1-\dots-w_{k-1}p_{k-1})\}$.
The nuisances are $p_1,\dots,p_{k-1}$, $p(\x| y=1)$,\dots, $p(\x| y=k)$, and $\pi$.
Then, the perpendicular space $\Lambda^{\perp}$, which is orthogonal to all the nuisance tangent spaces, is given by
\begin{align}\label{eq:perpendicular}
&\Lambda^{\perp}=
\Bigg[
\frac{r}{\pi}\Bigg\{1-\sum_{i=1}^k \omega_i \mathrm{I}(y=i)\Bigg\} E(\mathbf{h} | y=k)
\!+\!\Bigg\{\!\frac{r}{\pi} \notag\\ 
\!\times\! & \sum_{i=1}^k \omega_i \mathrm{I}(y\!=\!i)
-\!\frac{1\!-\!r}{1\!-\!\pi}\Bigg\} \mathbf{h}(\mathbf{x})\!:\!E_t(\mathbf{h})\!=\!\0\in\mathbb{R}^{k-1}
\Bigg].
\end{align}
\end{theorem}
Letting $\varphi(\x,y,r;\bo)$ be an influence function, it satisfies $E\left\{\varphi(\x,y,r;\bo)\right\}=\0$ at $\bo=\bo_0$.
Because there is a one-to-one relationship between influence function and element in (\ref{eq:perpendicular}): $\varphi(\x,y,r;\bo)=\mathbf{C}\mathbf{g}(\x,y,r;\bo)$, where $\mathbf{C}$ is an invertible constant matrix and $\mathbf{g}\in\Lambda^{\perp}$.
We have
\begin{equation}\label{eq:influ-perp}
E\left\{\varphi(\x,y,r;\bo)\right\}=\0\Leftrightarrow E\left\{\mathbf{g}(\x,y,r;\bo)\right\}=\0.
\end{equation}
Equation~(\ref{eq:influ-perp}) implies that solving the empirical version of $E\left\{\varphi(\x,y,r;\bo)\right\}=\0$ is equivalent to solving the empirical version of $E\left\{\mathbf{g}(\x,y,r;\bo)\right\}=\0$.
So, we can solve $\sum_{i=1}^{n+m}\mathbf{g}(\x_i,y_i,r_i;{\bo})=\0$ to estimate $\bo$.
The estimating equation $\sum_{i=1}^{n+m}\mathbf{g}(\x_i,y_i,r_i;{\bo})=\0$ can be further written as
\begin{align}\label{eq:empirical-estimating-equation}
&\sum_{j=1}^{n+m}\Bigg[
\frac{r_j}{\pi}\Bigg\{1-\sum_{i=1}^k{\omega}_i \mathrm{I}(y_j=i)\Bigg\} E\{\mathbf{h}(\x) | y=k\}\notag\\
&+\Bigg\{\frac{r_j}{\pi} \sum_{i=1}^k{\omega}_i \mathrm{I}(y_j=i)-\frac{1-r_j}{1-\pi}\Bigg\} \mathbf{h}({\x}_j)
\Bigg]=\0.
\end{align}
To solve the equation, we can split (\ref{eq:empirical-estimating-equation}) into two parts:
\begin{equation}
\label{eq:two-equations}
\begin{cases}
\sum_{j=1}^{n+m}\frac{r_j}{\pi}\left\{1\!-\!\sum_{i=1}^{k}{\omega}_i\mathrm{I}(y_j\!=\!i)\right\}E\left\{\h(\x)|y\!=\!k\right\}\!=\!\0,\\
\sum_{j=1}^{n+m}\left\{\frac{r_j}{\pi}\sum_{i=1}^{k}{\omega}_i\mathrm{I}(y_j\!=\!i)\!-\!\frac{1-r_j}{1\!-\!\pi}\right\}\h(\x_j)\!=\!\0.
\end{cases}
\end{equation}
If both equations hold, then Equation~(\ref{eq:empirical-estimating-equation}) holds.

Because $E\left\{\h(\x)|y=k\right\}$ is a constant, the first equation of (\ref{eq:two-equations}) is equivalent to $\sum_{j=1}^{n}\left\{1-\sum_{i=1}^{k}\omega_i\mathrm{I}(y_j=i)\right\}=\0$.
This equation holds because we put the a constraint on $\widehat{\bo}_i$: $\sum_{i=1}^{k}\widehat{\omega}_i\widehat{p}_i=1$, where $\widehat{p}_i$ is the proportion of observations with label $y=i$ in the source dataset.
So, we only need to achieve the second equation of (\ref{eq:two-equations}), which directly leads to the moment-matching framework described below.
\begin{definition}{(\textbf{\textit{Moment Matching Framework}})}
We call the following equation the moment-matching framework for label shift adaptation:
\begin{equation}\label{eq:empirical-moment-matching}
\frac{1}{n}\sum_{j=1}^{n}\sum_{i=1}^{k}{\omega}_i\mathrm{I}(y_j=i)\h(\x_j)=\frac{1}{m}\sum_{j=n+1}^{n+m}\h(\x_j).
\end{equation}
Equation~\eqref{eq:empirical-moment-matching} is the empirical version of $E_s\left\{\omega(y)\h(\x)\right\}=E_t\left\{\h(\x)\right\}$.
Here $\h(\x)\in\mathbb{R}^{k-1}$ is a function that satisfies $E_t(\h)=\0$.
\end{definition}
In equation~(\ref{eq:empirical-moment-matching}), one first finds a mapping $\x\mapsto\mathbf{h}(\x)\in\mathbb{R}^{k-1}$, and then estimates the importance weights (i.e., $\bo$) by matching the first moment of $\mathbf{h}(\x)$ on the target domain (i.e., $E_t\left\{\mathbf{h}(\x)\right\}$) with a shifted first moment (i.e., $E_s\left\{\omega(y)\mathbf{h}(\x)\right\}$).

In the perpendicular space (\ref{eq:perpendicular}), we require $\mathbf{h}(\x)$ to satisfy $E_t\left\{\mathbf{h}(\x)\right\}=\0$.
However, we do not need this constraint when choosing $\mathbf{h}(\x)$: for any $\mathbf{h}(\x)\in\mathbb{R}^{k-1}$, we can use $\mathbf{h}^\ast(\x)=\mathbf{h}(\x)-E_t\left\{\mathbf{h}(\x)\right\}$ to construct an element $\mathbf{g}(\x,y,r)\in\mathbf{\Lambda}^{\perp}$ as $E_t\left\{\mathbf{h}^\ast(\x)\right\}=\0$.
Because we are subtracting a constant on both sides of (\ref{eq:empirical-moment-matching}), using $\h^\ast(\x)$ yields the same estimator as $\h(\x)$.

\begin{remark}
Different choices of $\h(\x)$ lead to different estimators, while a poor choice (e.g., $\h(\x)$ is degenerate) may fail to yield a valid estimator.
It is worth noting that if we let $\mathbf{h}_{\mathrm{BBSE}}(\x)=[{p}_s(y=1|\x),\cdots,{p}_s(y=k-1|\x)]^\top$, the resulting estimator $\widehat{\bo}$ is essentially equivalent to the BBSE-soft (see \citealt{lipton2018detecting}) method when we replace $p_s(y|\x)$ with the estimated $\widehat{p}_s(y|\x)$ (but subject to the constraint that $\sum_{i=1}^{k}\widehat{p}_i\widehat{\omega}_i=1$).
\end{remark}


\section{Efficient Label Shift Adaptation}
\label{sec:proposed-estimator}
The BBSE method is easy to implement as the core step is only inverting a matrix.
But the performance of the BBSE method is not as good as the MLLS method proposed in \citet{alexandari2020}.
However, the MLLS method requires an extra step of calibrating the predictive model.
There are many calibration procedures available, and it is difficult to choose the calibration procedure when different calibration methods give different results.
Moreover, implementing the calibrated MLLS method is more difficult than the BBSE method.

The question is: Can we find an alternative method based on the perpendicular space (\ref{eq:perpendicular}) so that it can achieve good performances and is still easy to implement and computationally efficient?
We answer this question by proposing the ELSA estimator, which does rely on calibration and is computationally simple and efficient.

\subsection{The Estimation Method}

We propose $\mathbf{h}_{\mathrm{ELSA}}(\x)$  by
\begin{equation}\label{eq:hp}
\mathbf{h}_{\mathrm{ELSA}}(\x)=\left\{\frac{E_s\left\{\omega^2(y) | \x\right\}}{\pi}+\frac{E_s\{\omega(y) | \x\}}{1-\pi}\right\}^{-1}\widetilde{\boldsymbol{\mu}}(\x) ,
\end{equation}
where $\widetilde{\boldsymbol{\mu}}(\x)\coloneqq[p_s(y=1|\x)-p_s(y=k|\x),\dots,p_s(y=k-1|\x)-p_s(y=k|\x)]^{T}$.
The proposed $\mathbf{h}_{\mathrm{ELSA}}(\x)$  is inspired by projecting the score function $\mathbf{S}_{\bo}\coloneqq\partial\log\{\mathcal{L}(\x,y,r)\}/\partial\bo^{(-k)}$ onto the perpendicular space in (\ref{eq:perpendicular}).
The projection has a form $\mathbf{c}(\x)/[{E_s\left\{\omega^2(y) | \x\right\}}/{\pi}+{E_s\{\omega(y) | \x\}}/{(1-\pi)}]$, where $\c(x)$ is a function depends on $\x$ with complex structure.
We replace $\mathbf{c}(\x)$ with the $\widetilde{\boldsymbol{\mu}}(\x)$ for simple implementation.
More details on the motivation of $\mathbf{h}_{\mathrm{ELSA}}(\x)$ are given in Section~\ref{appendx-motivation} the appendix.


In practice, we can adopt a fitted model $\widehat{p}_s(y=j|\x)$ to replace $p_s(y=j|\x)$.
Then we have $\widehat{E}_s\{\omega^2(y)|\x\}=\sum_{i=1}^{k}\omega_i^2 \widehat{p}_s(y=i|\x)$, $\widehat{E}_s\{\omega(y)|\x\}=\sum_{i=1}^{k}\omega_i\widehat{p}_s(y=i|\x)$.
Also we have $\omega_k=(1-\sum_{i=1}^{k-1}\widehat{p}_i\omega_i)/\widehat{p}_k$, where $\widehat{p}_i$ ($i=1,\dots,k$) is the proportion of source data with label $y=i$.
Directly from (\ref{eq:empirical-moment-matching}), we can estimate $\bo$ by solving the following equation:
\begin{align}\label{eq:empirical-equation}
\frac{1}{n}\Big\{\!\mathbf{H}(\bo)\mathbf{A}\bo^{(-k)}\!+\!\mathbf{H}(\bo)\mathbf{v}\!\Big\}
\!\!-\!\frac{1}{m}\mathbf{H}^\ast(\bo)\mathbf{1}_{m\times1}\!=\!\0,
\end{align}
where $\bo^{(-k)}\coloneqq(\omega_1,\dots,\omega_{k-1})$,
$\mathbf{H}(\bo)=[\widehat{\mathbf{h}}_{\mathrm{ELSA}}(\x_1),\dots,\widehat{\mathbf{h}}_{\mathrm{ELSA}}(\x_n)]$,
$\mathbf{H}^\ast(\bo)=[\widehat{\mathbf{h}}_{\mathrm{ELSA}}(\x_{n+1}),\dots,\widehat{\mathbf{h}}_{\mathrm{ELSA}}(\x_{n+m})]$, $\mathbf{A}$ is a $n\times(k-1)$ matrix and the $j$th row of $\mathbf{A}$ corresponds to the $j$th observation in the source data and is defined as
\[
\mathbf{A}_{j\cdot}=
\begin{cases}
\left[1,0,\dots,0\right],\quad y_j=1,\\
\left[0,1,\dots,0\right],\quad y_j=2,\\
\ldots\\
\left[0,\dots,0,1\right],\quad y_j=k-1,\\
\left[-\widehat{p}_1/\widehat{p}_k,\dots, -\widehat{p}_{k-1}/\widehat{p}_k\right],\quad y_j=k,
\end{cases}
\]
and $\mathbf{v}$ is a vector of length $n$ whose $j$th entry $v_j=1/\widehat{p}_k$ if $y_j=k$, otherwise $v_j=0$ for $j=1,\dots,n$.
Here $\widehat{\mathbf{h}}_{\mathrm{ELSA}}(\x)$ is obtained by replacing $p_s(y|\x)$ in $\mathrm{h}_{\mathrm{ELSA}}(\x)$ with a fitted $\widehat{p}_s(y|\x)$.

\begin{remark}
We can find the solution $\hat{\bo}^{(-k)}$ to (\ref{eq:empirical-equation}) by minimizing the $\ell_2$-norm of the left side of (\ref{eq:empirical-equation}):
\begin{align}\label{Eq: min_prob}
\Big\|\frac{1}{n}\Big\{\!\mathbf{H}(\bo)\mathbf{A}\bo^{(-k)}\!+\!\mathbf{H}(\bo)\mathbf{v}\!\Big\}
\!\!-\!\frac{1}{m}\mathbf{H}^\ast(\bo)\mathbf{1}_{m\times1}\Big\|^2.
\end{align}
The problem~\eqref{Eq: min_prob} can be solved with any optimizer.
Another approach is to find the solution $\bo$ by the fix-point iteration: Given $\bo^{(-k)}_{t}=(\omega_1,\dots,\omega_{k-1})_{t}$ at the $t$-th iteration, we can update it by
$\bo^{(-k)}_{t+1}=\left\{\mathbf{H}(\bo_{t})\mathbf{A}\right\}^{-1}\left\{({n}/{m})\mathbf{H}^\ast(\bo_{t})\mathbf{1}_{m\times1}-\mathbf{H}(\bo_{t})\mathbf{v}
\right\}$. 
This method only involves matrix inversion and we can obtain the solution to (\ref{eq:empirical-equation}) when $\{\bo^{(-k)}_t\}$ converges.
\end{remark}

\subsection{Theoretical Results}

In this part, we will establish the theoretical propriety of our ELSA estimator.
It can be shown that our estimator $\widehat{\bo}^{(-k)}$ (or equivalently $\widehat{\bo}$) is a RAL estimator and thus enjoy the asymptotic normality.
We state the details in the following theorem.
The proofs are given in Section~\ref{sup:proofs} of the appendix.
\begin{theorem}\label{th:asymp}
Letting
\begin{align*}
    \mathbf{g}_p(\x,y,r)=\frac{r}{\pi}\left\{1-\sum_{i=1}^k \omega_i \mathrm{I}(y=i)\right\} E(\mathbf{h}_{\mathrm{ELSA}} | y=k)\notag\\
    +\left\{\frac{r}{\pi} \sum_{i=1}^k \omega_i \mathrm{I}(y=i)-\frac{1-r}{1-\pi}\right\} \mathbf{h}_{\mathrm{ELSA}}(\mathbf{x}),
\end{align*}
where $\mathbf{h}_{\mathrm{ELSA}}(\x)$ is defined in (\ref{eq:hp}).
Under the assumption that $E\bigg\{{\partial\mathbf{g}_p}/{{\partial\bo^{(-k)}}^{\mathrm{T}}}\bigg\}\big|_{\bo^{(-k)}=\bo^{(-k)}_0}$ is non-singular and $n/(n+m)\to\pi\in(0,1)$ as $n\to\infty$, the proposed estimator $\widehat{\bo}^{(-k)}$, as the solution to (\ref{eq:empirical-equation}), enjoys the following proprieties:
\begin{enumerate}
\item ($\sqrt{n}$-consistent) $\widehat{\bo}^{(-k)}={\bo}^{(-k)}_0+o_p({1}/{\sqrt{n}})$, i.e. for any small $\epsilon>0$, we can find a value $N$ such that for all $n>N$, $\Pr(\sqrt{n}\Vert\widehat{\bo}^{(-k)}-\bo_0^{(-k)}\Vert>\epsilon)<\epsilon$.
\item (asymptotic normality) as $n\to\infty$, it holds that
$\sqrt{n}\left(\widehat{\bo}^{(-k)}-\bo_0^{(-k)}\right)\xrightarrow{d}
\mathcal{N}\left(\mathbf{0}, \pi\mathbf{U}\mathbf{V}\mathbf{U}^\top\right)$.
\end{enumerate}
where $\mathbf{U}=\bigg[E\bigg\{{\partial\mathbf{g}_p}/{{\partial\bo^{(-k)}}^{\mathrm{T}}}\bigg\}\big|_{\bo^{(-k)}=\bo^{(-k)}_0}\bigg]^{-1}$ and $\mathbf{V} =E\{\mathbf{g}_p(\x,y,r;\bo_0^{(-k)})\mathbf{g}_p^{\mathrm{T}}(\x,y,r;\bo_0^{(-k)})\}$.
\end{theorem}
\begin{remark}
Note that the estimator $\widehat{\bo}^{(-k)}$ does not include $\omega_k$.
As $n\to\infty$, $\widehat{p}_j$ $(j=1,\dots,k)$ also converge to the true values $p_j$ $(j=1,\dots,k)$.
The estimator for the importance weight for class $y=k$ (i.e., $\widehat{\omega}_k$), which is given by $\widehat{\omega}_k=(1-\sum_{i=1}^{k-1}\widehat{p}_i\widehat{\omega}_i)/\widehat{p}_k$ is also consistent by the Slutsky's theorem.
\end{remark}
\begin{remark}
    The proposed estimator belongs to the family of the Z-estimator, and the conditions in Theorem 3.2 are standard regularity assumptions for the Z-estimator. More details on the Z-estimator and its regularity assumptions can be found in Chapter 5 in \citet{vaart_1998}.
\end{remark}

\section{Numerical Studies}\label{sec:numerical-studies}

\begin{figure*}[t!]
\centering
\begin{tabular}{@{}ccc@{}}
\includegraphics[width=0.31\linewidth]{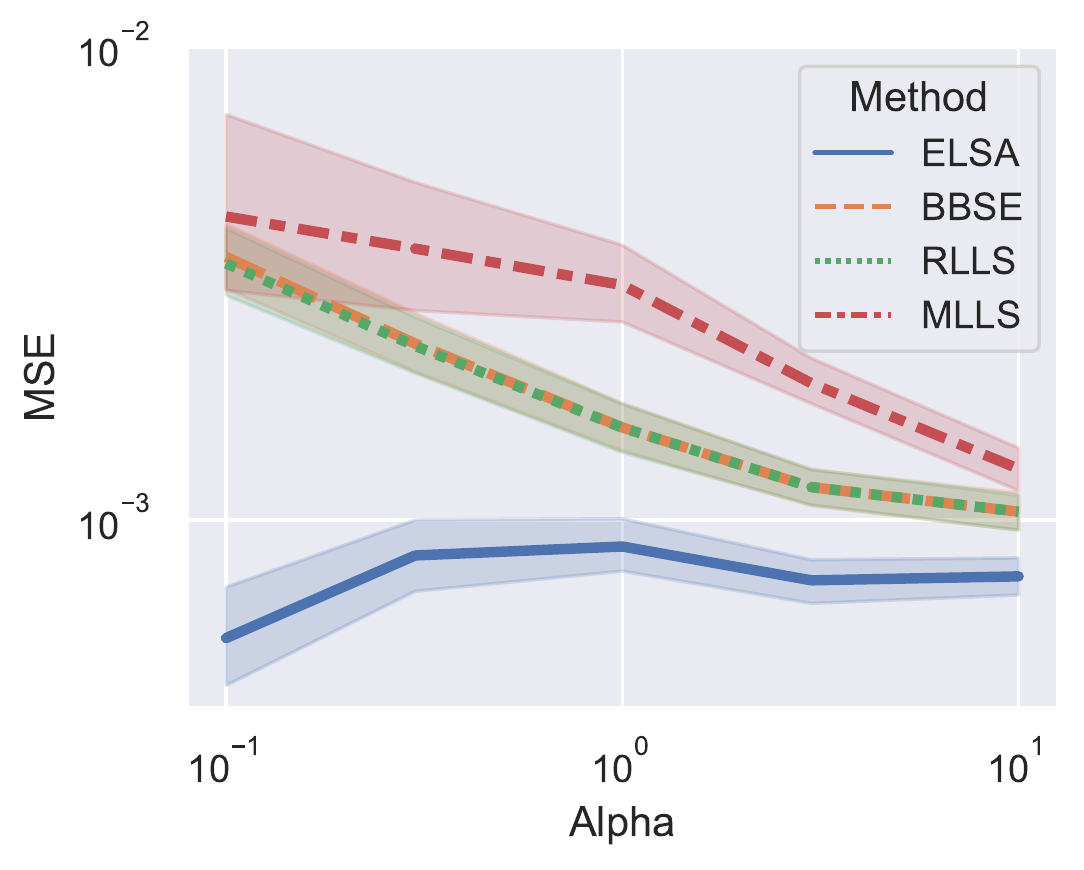}&
\includegraphics[width=0.31\linewidth]{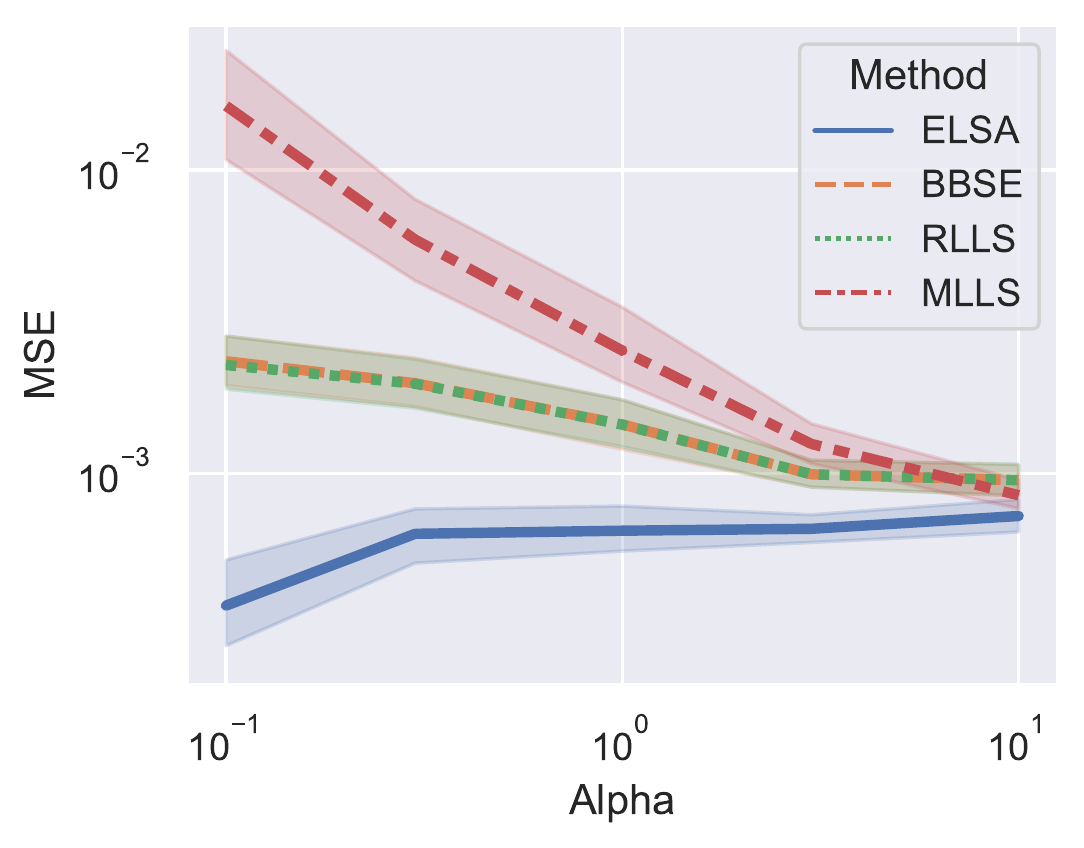}&
\includegraphics[width=0.31\linewidth]{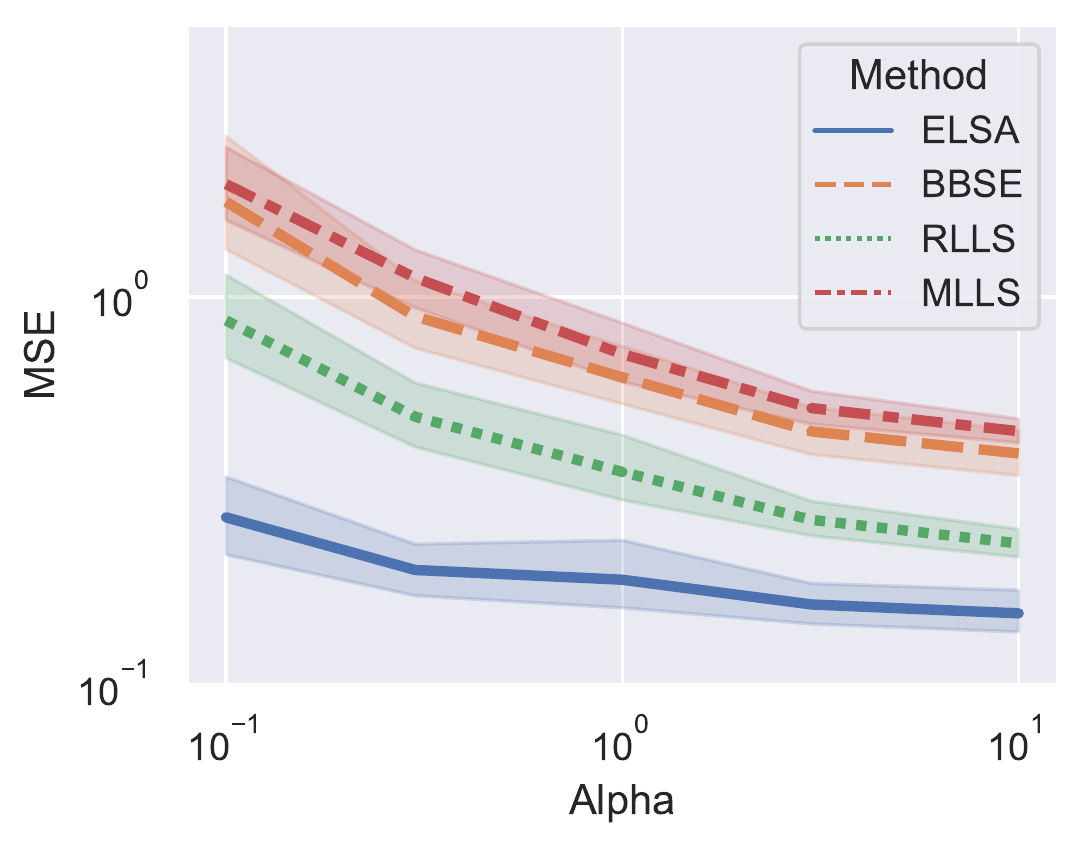}\\
~~~~(a) MNIST & ~~~~(b) CIFAR-10 & ~~~~(c) CIFAR-100\\
\end{tabular}
\caption{MSE v.s. the degree of shift. We control the label shift by adjusting the Dirichlet shift parameter $\alpha$. We fix the sample size as $5000$. The solid curves represent the mean trimmed $5\%$ and the shadow regions are $95\%$ CI error band.}\label{fig.MSE_over_SampleSize}
\end{figure*}

\begin{figure*}[t!]
\centering
\begin{tabular}{@{}ccc@{}}
\includegraphics[width=0.31\linewidth]{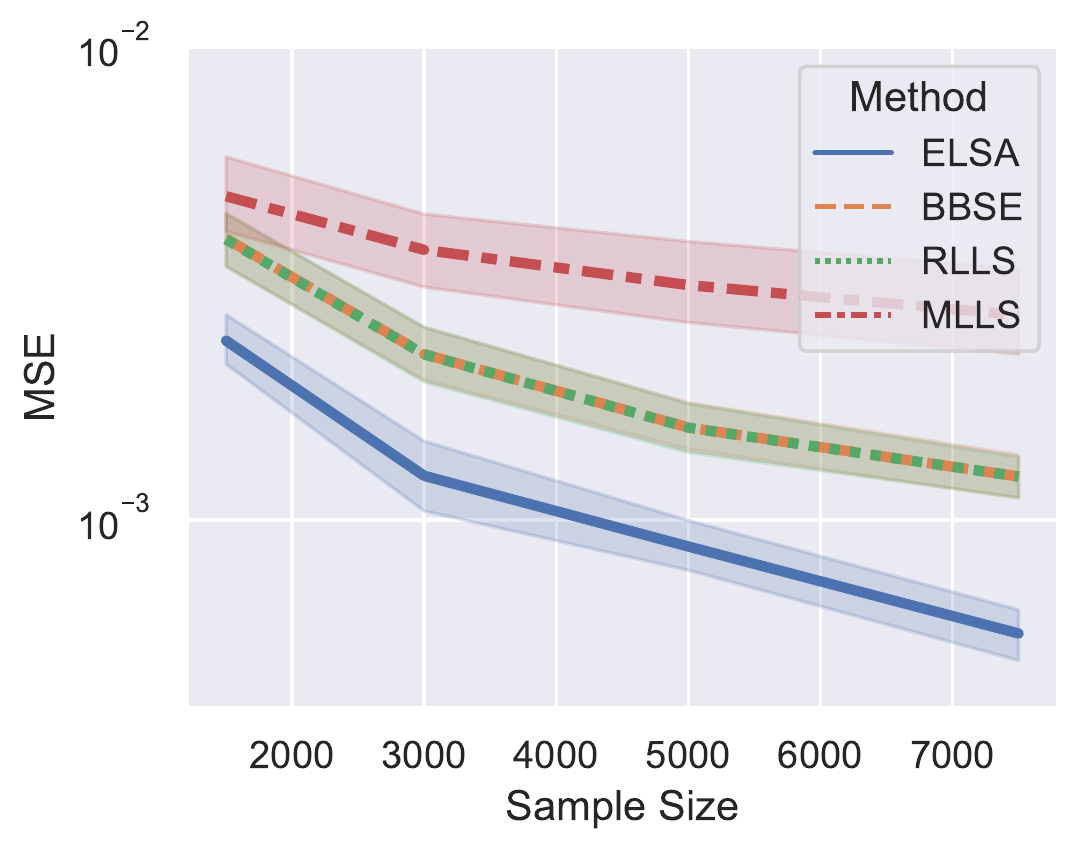}&
\includegraphics[width=0.31\linewidth]{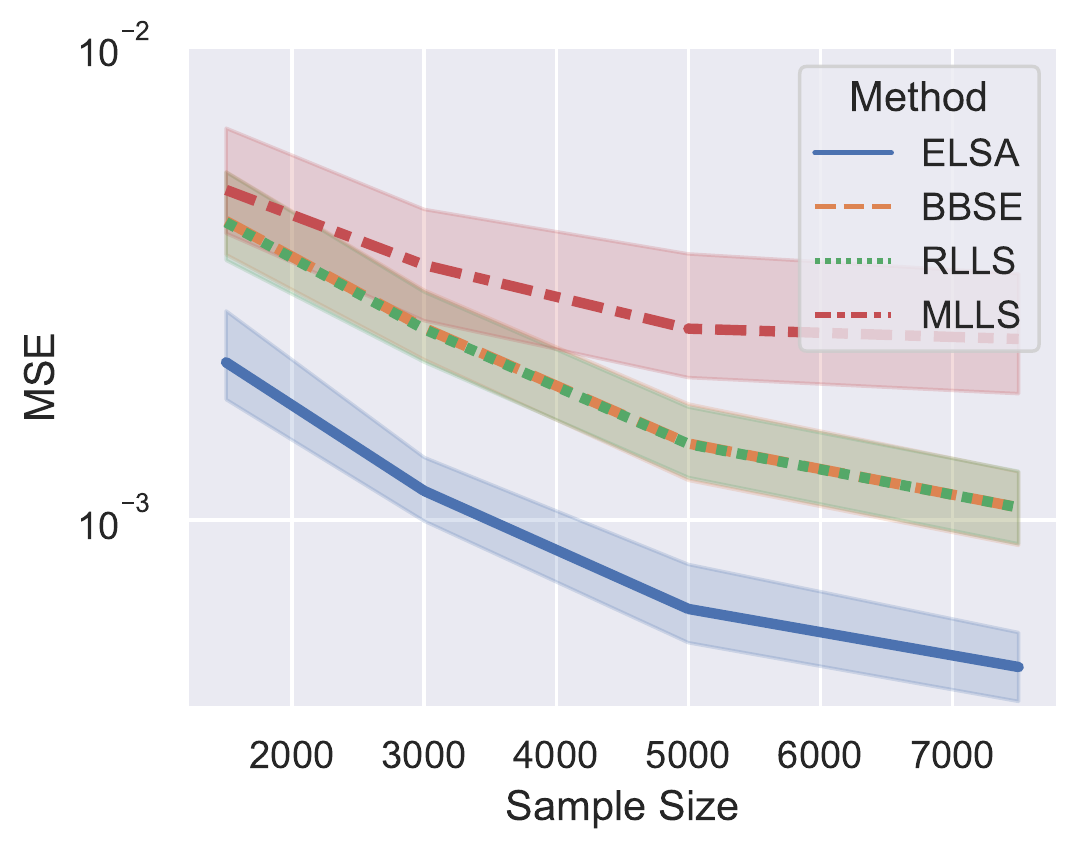}&
\includegraphics[width=0.31\linewidth]{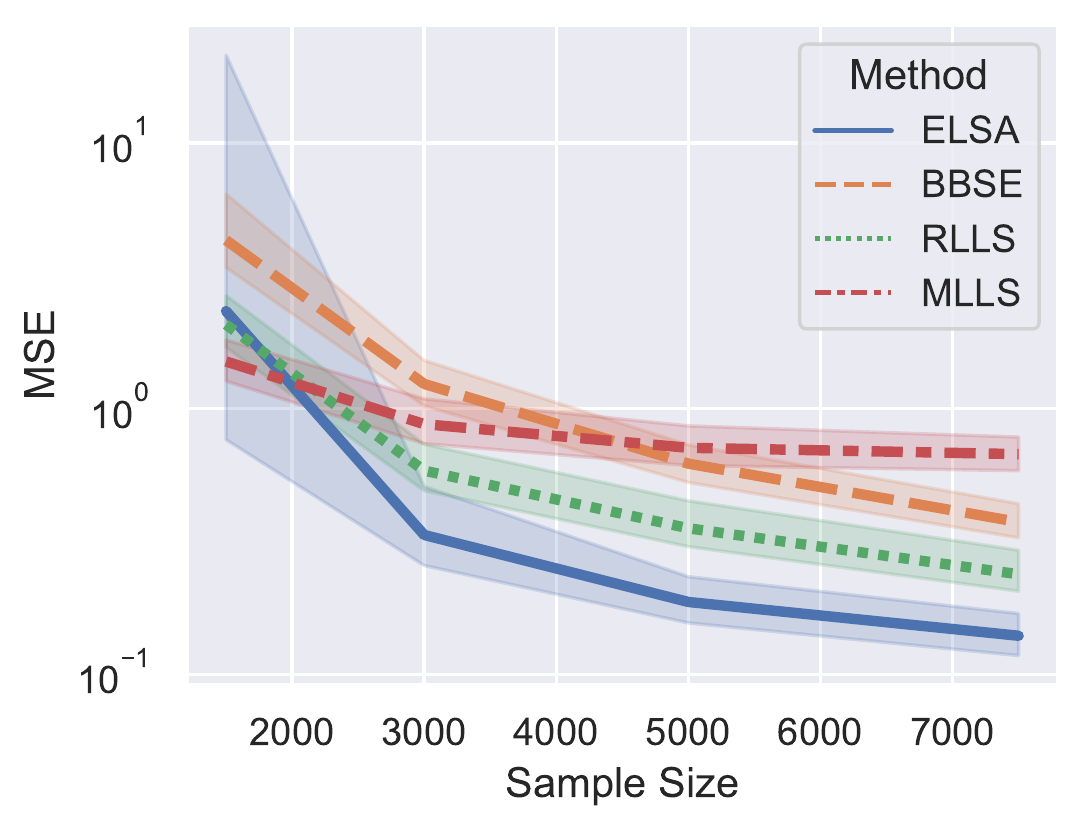}\\
~~~~(a) MNIST & ~~~~(b) CIFAR-10 & ~~~~(c) CIFAR-100\\
\end{tabular}
\caption{MSE v.s. the sample size. We fix the Dirichlet shift parameter $\alpha$ as $1.0$ for the three datasets. The solid curves represent the mean trimmed $5\%$, and the shadow regions are $95\%$ CI error band.}\label{fig.MSE_over_Alpha}
\end{figure*}

\begin{figure*}[t!]
\centering
\includegraphics[width=1.0\linewidth]{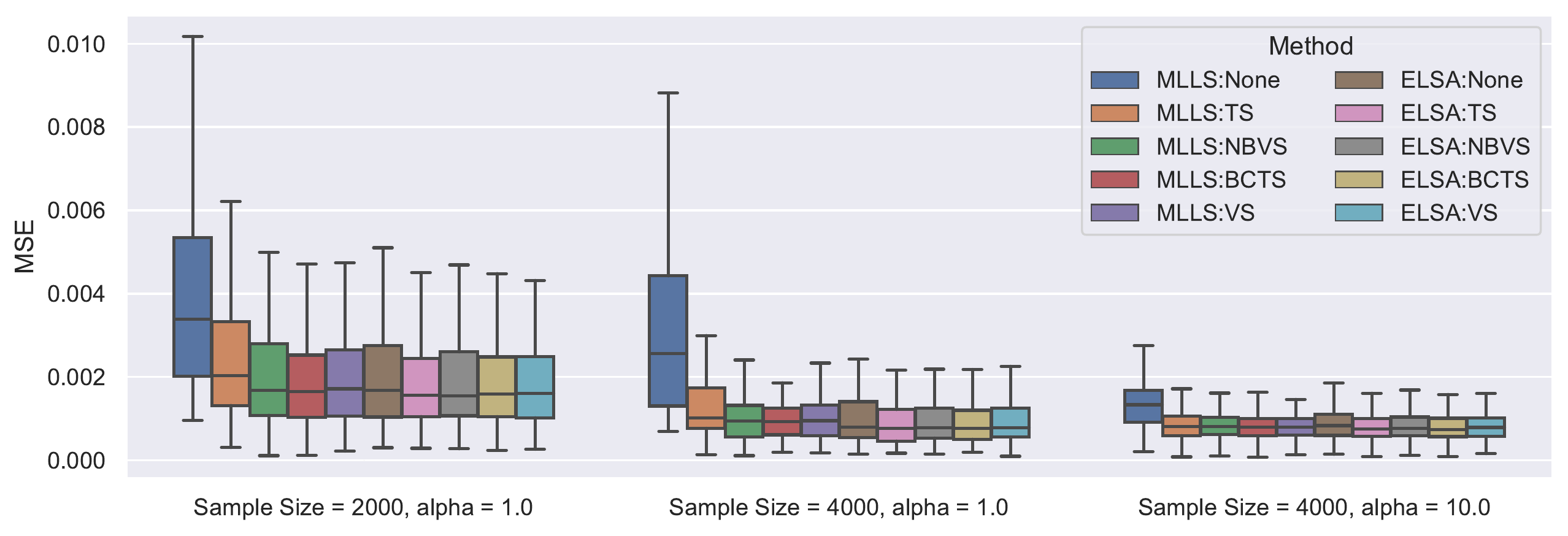}\\
(a) MNIST \\
\includegraphics[width=1.0\linewidth]{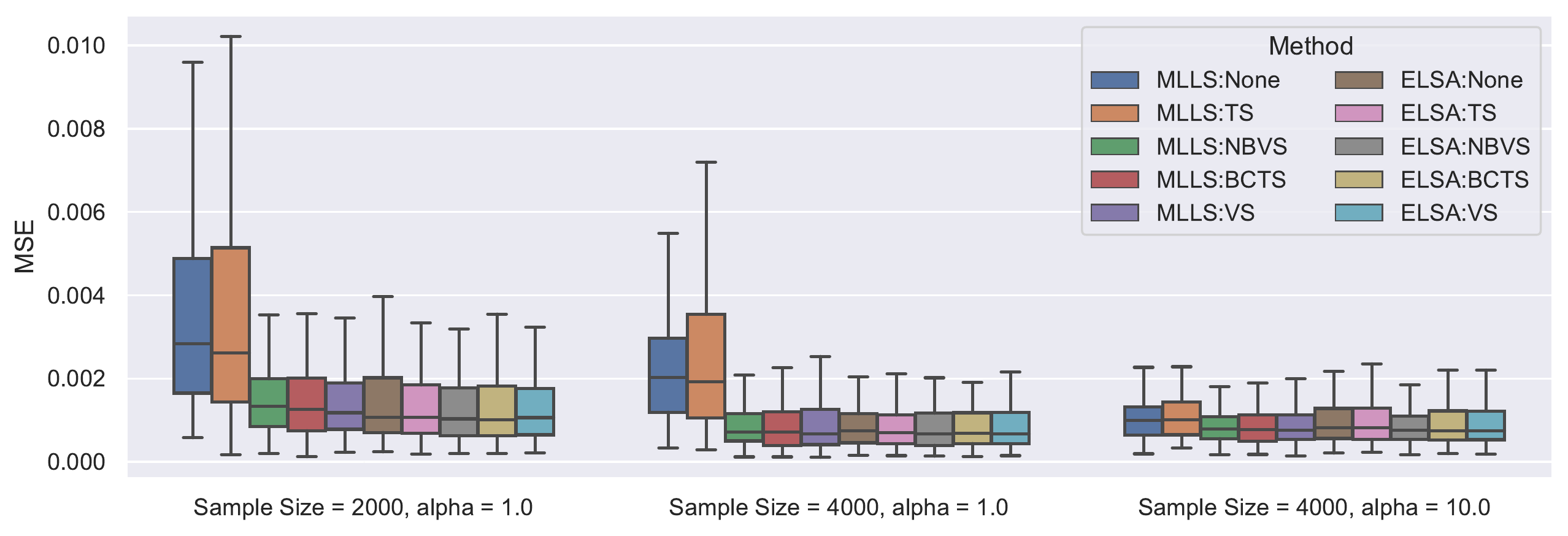}\\
(a) CIFAR-10 \\
\caption{The MSE boxplots under calibrations for ELSA and MLLS methods. }\label{fig.MSE_over_Cali}
\end{figure*}

In this section, we conduct numerical experiments to demonstrate the efficacy of our proposed method on the label shift problem. We summarize our experimental settings in the following.

\textbf{Datasets and Models}. 
Our experiments are evaluated on MNIST, CIFAR-10, and CIFAR-100. 
We adopt the same settings in \cite{alexandari2020}: 
For each dataset, ten models are trained with different random seeds, and 10k data samples of the training set are reserved as the source dataset (so that it is not used for training the model.)
The models for MNIST have the same architectures as the ones in \cite{azizzadenesheli2018regularized} and those for CIFAR-10 and CIFAR-100 are the same as the models in \cite{geifman2017selective}.

\textbf{Label Shift Mechanism}. 
Dirichlet label shift is adopted as our label shift mechanism.
More specifically, we use a Dirichlet distribution with concentration parameter $\alpha$ to generate the label distribution $p(y)$ for the target dataset.
Then we create the target dataset by sampling with replacement according to $p(y)$. The label shift is more severe with a smaller $\alpha$. We will call $\alpha$ as the Dirichlet shift parameter in the following studies. 

\textbf{Baseline Methods and Metric}. We compare our proposed method with three baseline methods: BBSE \cite{lipton2018detecting}, RLLS \cite{azizzadenesheli2018regularized} and the maximum likelihood method via an EM algorithm \cite{alexandari2020,garg2020} (we call it the MLLS method in the rest of the paper). Furthermore, we consider boosting the performance of the MLLS method by adopting prediction calibrations. To evaluate the performance, we adopt the mean squared error (MSE) between the true weights and the estimated weights as the metric. In our experiments, for all the settings, we run 20 trials for each model and, thus, 200 replications in total. 

Due to the space limitation, in this section, we will focus on the following three questions:
\vspace{-.15in}
\begin{enumerate} 
    \item How's the performance of our ELSA compared with the baseline methods under different scenarios?
    \vspace{-.1in}
    \item Would ELSA require calibrations?
    \vspace{-.1in}
    \item How much computation efficiency does ELSA gain?
\end{enumerate}
\vspace{-.15in}
More experiments and detailed results are relegated to Section~\ref{supp:simu} of the appendix.

\subsection{ELSA Performs {\em Better} than Baseline Approaches}

Our first experimental study focuses on the efficacy of our proposed method under different sample sizes and shift parameters. We mainly compare our proposed method with the standard baseline methods, BBSE, RLLS, and MLLS. In all the experiments, we keep the source and target datasets with the same sample size and generate the two sets by sampling with replacement.

Figure~\ref{fig.MSE_over_SampleSize} illustrates the estimation results of our proposed method on the three datasets under the different Dirichlet shift parameters. We fix the sample size as $n=m=5000$ while adjust $\alpha$ from $\{0.1, 0.3, 1.0, 3.0, 10.0\}$. The figures show that our method has the most accurate estimation in terms of the MSE. 
For example, in the case of MNIST with $\alpha=1.0$, our method reduces the MSE from RLLS's $1.567\times 10^{-3}$ to $8.78\times10^{-4}$, which is almost $50\%$ improvement.
Also, the MSE trends of our method are relatively flat over different $\alpha$: The MSE values are around $10^{-3}$ for MNIST and CIFAR-10, and about $10^{-1}$ for CIFAR-100. However, the MSE gets larger for the other methods while $\alpha$ is closer to zero. 

We also compare our method with the others on different sample sizes. We fix $\alpha=1.0$ and allow the sample size $n=m$ ranging from $\{1500, 3000, 5000, 7500\}$. The results are summarized in Figure~\ref{fig.MSE_over_Alpha}.
Our method consistently has a lower MSE than the other three methods on MNIST and CIFAR-10. For CIFAR-100, the estimation performance of our method on $1500$ samples is similar to the competitive methods. But as the sample size increases, our method performs better than the others.

\subsection{Calibration is {\em Not} Necessary for ELSA}

As it is studied in \cite{alexandari2020, garg2020}, post-prediction calibration is an essential step for the success of the MLLS method.
In this study, we want to examine whether calibration is necessary for our method.
Here we use calibration methods from temperature scaling (TS), no-bias vector scaling (NBVS), bias-corrected temperature scaling (BCTS), and vector scaling (VS). We refer the readers to \cite{alexandari2020} for details of the calibration methods. Also, we consider a none calibration case in which only soft-max operation is adopted to map the prediction score into $[0,1]$.

We study the combination of our proposed and the MLLS methods over the five calibration methods. 
Figure~\ref{fig.MSE_over_Cali} shows the corresponding estimation performance under the datasets of MNIST and CIFAR-10. In the experiments, we consider the sample size ($n=m$) and $\alpha$ pairs in $(2000, 1.0)$, $(4000, 1.0)$, and $(4000, 10.0)$.
It can be seen that the MLLS method with the none calibration has the poorest estimation result.
By performing the proper calibration, the MLLS method's estimation could improve.
For example, in the case of CIFAR-10 with sample size $n=m=4000$ and $\alpha=1.0$, the median MSE of the MLLS method with the none calibration is $2.024\times 10^{-3}$ while reduced to about $7 \times 10^{-4}$ by adding the BCTS calibration.
However, the performance of our method is very stable under different calibrations and data settings.
Even with the none calibration, our method has a competitive estimation as the best result from the MLLS method.
Thus, we conclude that the post-prediction calibration is unnecessary for our proposed method.

\subsection{ELSA Achieves the High Accuracy {\em Fast}}

\begin{figure}[t!]
\centering
\includegraphics[width=0.95\linewidth]{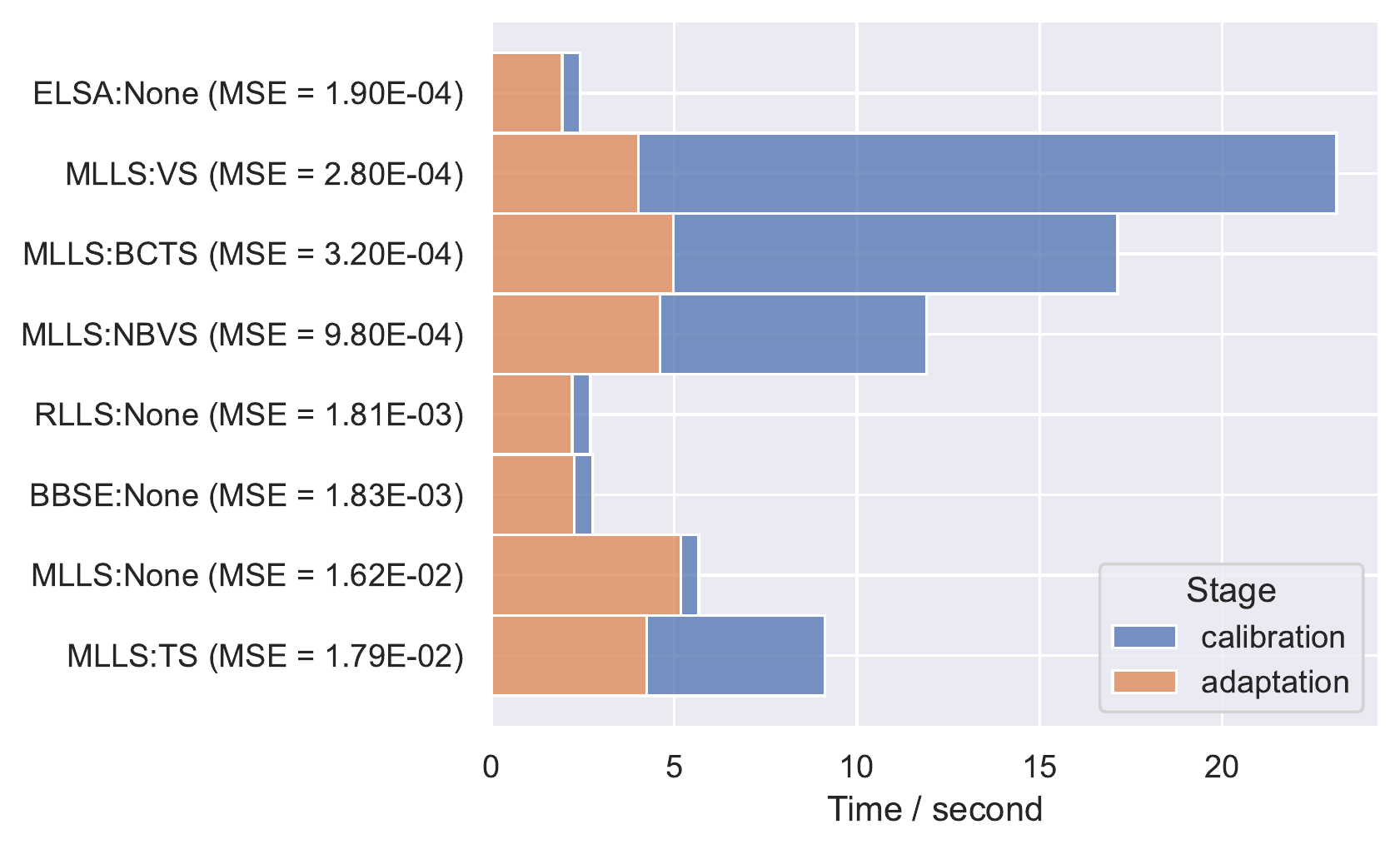}
\vspace{-.2in}
\caption{The stacked barplot for the computation times of different methods on CIFAR-10. The Dirichlet shift parameter $\alpha$ and sample size are fixed as 0.1 and 10000, respectively. The MSE values are the mean trimmed $5\%$.}
\label{fig.Method_over_Time}
\end{figure}

In the end, we aim to examine the computation efficiency of our method.
The experiments are conducted on a MacBook Pro with a 2.9 GHz Dual-Core Intel Core i5 processor and 8 GB memory.
We compare our method with the three baseline methods, BBSE, RLLS, and MLLS, without calibration. Additionally, for the MLLS method, we also consider its combinations with TS, VS, BCTS, and NBVS calibrations.
We study the case with a dataset as CIFAR-10, sample size $m=n=10000$, and $\alpha=0.1$.

Figure~\ref{fig.Method_over_Time} illustrates the trade-off between computation and estimation over different methods.
We rank the methods based on their trimmed mean MSE after removing $5\%$ outliers, and the upper method is with smaller MSE.
We used the stacked barplot to show the accumulated computation time, separated into the stages of post-prediction calibration and label shift adaptation. 
Note that for all the methods with none calibration, the time at the calibration stage is for the data reformatting and the soft-max operation.
It can be seen that the top 3 methods (our ELSA, the MLLS method with VS, and BCTS calibrations) have the MSE values reduced by about $70\%$ compared with the 4th method (the MLLS method with NBVS calibration). 
Within the top 3 methods, our method uses the shortest time (less than 2.5 seconds) to finish the whole process, while the MLLS method with VS calibration costs over 20 seconds and with BCTS costs about 10 seconds.
Thus, our method could achieve high accuracy with the fastest computation.


\section{Discussion}\label{sec:discussion}

We propose a moment-matching framework for adapting the label shift for constructing an asymptotic linear estimator.
Furthermore, we propose the ELSA estimator by using a novel mapping for $\x$ (i.e., $\mathbf{h}_p(\x)$) in the moment-matching procedure; the resulting estimator has a competitive performance even without calibrating the predictive models.
The proposed method is also simple to implement and computationally efficient.
So, it is especially useful when calibration is difficult (e.g., different calibration methods do not agree) or one needs to frequently adapt to label shift.
For future research, one research question would be to find the optimal mapping function for $\x$.
Other research questions include extending the semiparametric framework to a more general label shift setting like the open set label shift (see \citealt{garg2022domain}).

\bibliography{refPARLABELSHIFT}
\bibliographystyle{icml2023}

\newpage
\appendix
\onecolumn
\section*{Appendix}

\section{Background}
\label{sup:pre}
Here we briefly introduce the asymptotically linear estimator and the influence function.
Suppose we have independent and identically distributed data $\mathbf{Z}=\left\{Z_1,\dots,Z_n\right\}$, and the goal is to estimate $\bb\in\mathbb{R}^{q}$.
If there exists a $q$-dimensional random vector $\varphi^{q\times1}(Z)$ such that
\begin{align}
n^{1/2}\left(\widehat{\bb}_n-\bb_0\right)=n^{-1/2}\sum_{i=1}^{n}\varphi(Z_i)+o_p(1),
\end{align}
then we call $\widehat{\bb}_n$, as a function $\mathbf{Z}$, is an asymptotic linear estimator, and we say $\varphi(Z)$ is the \textit{influence function} of $\widehat{\bb}_n$.
For example, if $\bb=E(Z)$, then the sample mean $\widehat{\bb}_n=\sum_{i=1}^{n}Z_i/n$ is an asymptotic linear estimator with influence function $\varphi(Z)=Z-\bb$.
For an asymptotically linear estimator, using the central limit theorem, we have
$n^{1/2}\left(\widehat{\bb}_n-\bb_0\right)\xrightarrow{d}\mathcal{N}(\textbf{0},E(\varphi\varphi^T))$, given that $E(\varphi\varphi^T)$ is finite nonsingular.
Generally, there is a one-to-one relationship between the influence function and the corresponding asymptotically linear estimator.
Given an influence function, we can obtain the corresponding estimator by solving 
\begin{equation}\label{eq:estimating-equation}
\sum_{i=1}^{n}\varphi(z_i;\widehat{\bb}_n)=\0,
\end{equation}
which is the empirical version of $E\left\{\varphi(Z;\bb)\right\}=\0$.
So, we can construct an asymptotic linear estimator by finding an influence function.

To find an influence function, we can start with the Hilbert space of the mean-zero functions of the same length of the parameter $\bb$.
In other words, the whole space of all possible functions.
But we want to further narrow down from the whole space to a subspace that contains desirable candidates.
To achieve this goal, we need to partition the whole space using the \textit{nuisance tangent space}.

Here we discuss the notation of nuisance tangent space non-technically.
The likelihood function can be decomposed into two parts: for one part, we use parametric models; for the other part, we do not assume parametric models (i.e., we use non-parametric models).
For each nonparametric model, we can derive an individual nuisance tangent space; thus, after deriving all the individual nuisance tangent spaces for each nonparametric model, we combine them together to form a nuisance tangent space for the semiparametric model.
The nuisance tangent space is a subspace of the whole Hilbert space.

After deriving the nuisance tangent space, an estimator of the parameters (that belong to the parametric models) correlates with the nuisance tangent space in the following manner: To be an influence function, a function needs to be orthogonal to the {nuisance tangent space}.
Thus, the \textit{perpendicular space} to the nuisance tangent space contains all the influence functions.
Under some regularity conditions, we can choose an element from the perpendicular space, which can serve as the influence function.

In our strategy to address the label shift problem, we aim to find an asymptotic linear estimator for the importance weight $\bo=(\omega_1,\dots,\omega_k)$.
To achieve this goal, we first derive the nuisance tangent space.
Then, we derive the perpendicular space, which is orthogonal to the nuisance tangent space.
We will see that the perpendicular space leads to a moment-matching framework for addressing the label shift.
The basis for all the aforementioned derivations is the semiparametric likelihood, which is the focus of the next subsection.

\begin{table}[ht]
    \centering
        \caption{Method comparison on MNIST with Dirichlet shift. The value is the mean MSE trimmed $5\%$ extreme values. The bold values under the same setting (column) are the top-3 results.} 
    \label{table.MINIST_Dirichlet}
\begin{tabular}{llllllll}
\toprule
 \multirow{2}*{Calibration}  & \multirow{2}*{Adaptation}  & \multicolumn{6}{c}{Setting: (Sample Size $n$, Shift Parameter $\alpha$)} \\
  \cline{3-8}
&  &          (500, 0.1) &          (500, 1.0) &         (1500, 0.1) &         (1500, 1.0) &         (4500, 0.1) &         (4500, 1.0) \\
\midrule
 \multirow{6}*{None} & BBSE-hard &           2.048E-02 &           1.260E-02 &           7.301E-03 &           3.937E-03 &           3.841E-03 &           1.661E-03 \\
   & RLLS-hard &           1.886E-02 &           1.250E-02 &           6.932E-03 &           3.921E-03 &           3.721E-03 &           1.659E-03 \\
   & BBSE-soft &           1.744E-02 &           1.050E-02 &           6.997E-03 &           3.178E-03 &           3.748E-03 &           1.334E-03 \\
   & RLLS-soft &           1.618E-02 &           1.039E-02 &           6.749E-03 &           3.178E-03 &           3.588E-03 &           1.335E-03 \\
   & MLLS &           2.312E-02 &           1.167E-02 &           9.177E-03 &           4.839E-03 &           4.661E-03 &           3.156E-03 \\
   & ELSA &           3.588E-03 &           7.474E-03 &           1.610E-03 &           2.394E-03 &           6.020E-04 &           9.317E-04 \\
\midrule
 \multirow{4}*{BCTS} & BBSE-soft &           1.733E-02 &           1.041E-02 &           6.636E-03 &           3.113E-03 &           3.352E-03 &           1.264E-03 \\
   & RLLS-soft &           1.653E-02 &           1.033E-02 &           6.379E-03 &           3.113E-03 &           3.237E-03 &           1.265E-03 \\
   & MLLS &  \textbf{3.556E-03} &  \textbf{6.290E-03} &  \textbf{1.456E-03} &  \textbf{2.235E-03} &           7.500E-04 &           1.016E-03 \\
   & ELSA &           3.612E-03 &           7.284E-03 &           1.674E-03 &           2.342E-03 &  \textbf{5.876E-04} &  \textbf{8.419E-04} \\
\midrule
 \multirow{4}*{NBVS} & BBSE-soft &           1.761E-02 &           1.055E-02 &           6.833E-03 &           3.177E-03 &           3.395E-03 &           1.286E-03 \\
   & RLLS-soft &           1.686E-02 &           1.047E-02 &           6.587E-03 &           3.178E-03 &           3.277E-03 &           1.287E-03 \\
   & MLLS &           7.056E-03 &           7.197E-03 &           2.428E-03 &           2.459E-03 &           1.473E-03 &           1.043E-03 \\
   & ELSA &  \textbf{3.587E-03} &           7.438E-03 &           1.619E-03 &           2.408E-03 &  \textbf{5.917E-04} &  \textbf{8.676E-04} \\
\midrule
 \multirow{4}*{TS} & BBSE-soft &           1.638E-02 &           1.010E-02 &           6.439E-03 &           3.037E-03 &           3.389E-03 &           1.263E-03 \\
   & RLLS-soft &           1.520E-02 &           1.001E-02 &           6.240E-03 &           3.034E-03 &           3.270E-03 &           1.263E-03 \\
   & MLLS &           2.199E-02 &           9.362E-03 &           8.106E-03 &           2.894E-03 &           3.823E-03 &           1.347E-03 \\
   & ELSA &  \textbf{3.519E-03} &  \textbf{7.000E-03} &           1.635E-03 &  \textbf{2.292E-03} &  \textbf{5.919E-04} &  \textbf{8.503E-04} \\
\midrule
 \multirow{4}*{VS}  & BBSE-soft &           1.869E-02 &           1.044E-02 &           6.756E-03 &           3.141E-03 &           3.351E-03 &           1.296E-03 \\
   & RLLS-soft &           1.784E-02 &           1.037E-02 &           6.530E-03 &           3.141E-03 &           3.251E-03 &           1.296E-03 \\
   & MLLS &           4.639E-03 &  \textbf{6.672E-03} &  \textbf{1.535E-03} &  \textbf{2.339E-03} &           7.481E-04 &           1.030E-03 \\
   & ELSA &           3.681E-03 &           7.143E-03 &  \textbf{1.583E-03} &           2.364E-03 &           6.520E-04 &           8.682E-04 \\
\bottomrule
\end{tabular}
\end{table}

\begin{table}[ht]
    \centering
        \caption{Prediction improvement comparison on MNIST with Dirichlet shift. The value is the mean delta accuracy trimmed $5\%$ extreme values. The bold values under the same setting (column) are the top-3 results.} 
    \label{table.MINIST_Dirichlet_acc}
\begin{tabular}{llllllll}
\toprule
 & Settings & (500, 0.1) & (500, 1.0) & (1500, 0.1) & (1500, 1.0) & (4500, 0.1) & (4500, 1.0) \\
\midrule
 \multirow{6}*{None} & BBSE-hard & 5.598e-02 & 1.064e-02 & 5.614e-02 & 1.064e-02 & 5.736e-02 & 1.130e-02 \\
 & RLLS-hard & 5.600e-02 & 1.073e-02 & 5.610e-02 & 1.064e-02 & 5.738e-02 & 1.130e-02 \\
 & BBSE-soft & 5.709e-02 & 1.169e-02 & 5.639e-02 & 1.139e-02 & 5.763e-02 & 1.176e-02 \\
 & RLLS-soft & 5.713e-02 & 1.164e-02 & 5.643e-02 & 1.139e-02 & 5.768e-02 & 1.175e-02 \\
 & MLLS & 5.840e-02 & 1.038e-02 & 5.636e-02 & 9.563e-03 & 5.750e-02 & 9.889e-03 \\
 & ELSA & \textbf{5.922e-02} & 1.207e-02 & 5.737e-02 & 1.134e-02 & 5.857e-02 & 1.166e-02 \\
\cline{1-8}
 \multirow{4}*{BCTS} & BBSE-soft & 5.467e-02 & 1.093e-02 & 5.552e-02 & 1.296e-02 & 5.750e-02 & 1.417e-02 \\
 & RLLS-soft & 5.480e-02 & 1.093e-02 & 5.558e-02 & 1.296e-02 & 5.757e-02 & 1.417e-02 \\
 & MLLS & 5.709e-02 & 1.200e-02 & 5.711e-02 & 1.312e-02 & 5.844e-02 & 1.409e-02 \\
 & ELSA & 5.802e-02 & 1.180e-02 & \textbf{5.794e-02} & 1.317e-02 & \textbf{5.931e-02} & \textbf{1.437e-02} \\
\cline{1-8}
 \multirow{4}*{NBVS} & BBSE-soft & 5.447e-02 & 1.158e-02 & 5.537e-02 & 1.309e-02 & 5.723e-02 & 1.417e-02 \\
 & RLLS-soft & 5.451e-02 & 1.156e-02 & 5.542e-02 & 1.309e-02 & 5.729e-02 & 1.417e-02 \\
 & MLLS & 5.753e-02 & 1.238e-02 & 5.659e-02 & 1.313e-02 & 5.825e-02 & 1.413e-02 \\
 & ELSA & \textbf{5.871e-02} & 1.249e-02 & \textbf{5.763e-02} & 1.332e-02 & \textbf{5.930e-02} & \textbf{1.434e-02} \\
\cline{1-8}
 \multirow{4}*{TS} & BBSE-soft & 5.607e-02 & 1.384e-02 & 5.611e-02 & \textbf{1.348e-02} & 5.732e-02 & 1.402e-02 \\
 & RLLS-soft & 5.611e-02 & \textbf{1.387e-02} & 5.621e-02 & 1.347e-02 & 5.734e-02 & 1.402e-02 \\
 & MLLS & 5.827e-02 & \textbf{1.422e-02} & 5.714e-02 & \textbf{1.352e-02} & 5.820e-02 & 1.405e-02 \\
 & ELSA & \textbf{5.947e-02} & \textbf{1.460e-02} & \textbf{5.789e-02} & \textbf{1.356e-02} & 5.924e-02 & 1.422e-02 \\
\cline{1-8}
 \multirow{4}*{VS} & BBSE-soft & 5.158e-02 & 9.067e-03 & 5.479e-02 & 1.230e-02 & 5.742e-02 & 1.407e-02 \\
 & RLLS-soft & 5.169e-02 & 9.067e-03 & 5.482e-02 & 1.230e-02 & 5.746e-02 & 1.407e-02 \\
 & MLLS & 5.467e-02 & 9.800e-03 & 5.613e-02 & 1.253e-02 & 5.827e-02 & 1.411e-02 \\
 & ELSA & 5.598e-02 & 1.004e-02 & 5.723e-02 & 1.265e-02 & \textbf{5.953e-02} & \textbf{1.430e-02} \\
\cline{1-8}
\bottomrule
\end{tabular}
\end{table}    

\begin{table}[ht]
    \centering
        \caption{Method comparison on CIFAR-10 with Dirichlet shift. The value is the mean MSE trimmed $5\%$ extreme values. The bold values under the same setting (column) are the top-3 results.}
    \label{table.CIFAR-10_Dirichlet}
    
    \begin{tabular}{llllllll}
\toprule
 \multirow{2}*{Calibration}  & \multirow{2}*{Adaptation}  & \multicolumn{6}{c}{Setting: (Sample Size $n$, Shift Parameter $\alpha$)} \\
  \cline{3-8}
   &  &          (500, 0.1) &          (500, 1.0) &         (1500, 0.1) &         (1500, 1.0) &         (4500, 0.1) &         (4500, 1.0) \\
\midrule
 \multirow{6}*{None} & BBSE-hard &           2.508E-02 &           1.197E-02 &           8.058E-03 &           4.288E-03 &           2.614E-03 &           1.559E-03 \\
   & RLLS-hard &           2.240E-02 &           1.175E-02 &           7.271E-03 &           4.261E-03 &           2.522E-03 &           1.557E-03 \\
   & BBSE-soft &           2.073E-02 &           8.784E-03 &           6.237E-03 &           3.490E-03 &           2.317E-03 &           1.198E-03 \\
   & RLLS-soft &           1.881E-02 &           8.690E-03 &           5.770E-03 &           3.465E-03 &           2.255E-03 &           1.194E-03 \\
   & MLLS &           3.490E-02 &           9.183E-03 &           2.342E-02 &           4.991E-03 &           1.673E-02 &           2.565E-03 \\
   & ELSA &  \textbf{2.790E-03} &           5.801E-03 &  \textbf{8.887E-04} &           2.154E-03 &           3.851E-04 &           7.279E-04 \\
\midrule
 \multirow{4}*{BCTS} & BBSE-soft &           2.083E-02 &           8.705E-03 &           6.369E-03 &           3.420E-03 &           2.247E-03 &           1.106E-03 \\
   & RLLS-soft &           1.885E-02 &           8.613E-03 &           6.027E-03 &           3.395E-03 &           2.148E-03 &           1.100E-03 \\
   & MLLS &  \textbf{2.941E-03} &  \textbf{5.116E-03} &           9.181E-04 &           2.133E-03 &           5.229E-04 &           7.503E-04 \\
   & ELSA &           3.190E-03 &  \textbf{5.655E-03} &           9.579E-04 &  \textbf{2.080E-03} &  \textbf{3.748E-04} &  \textbf{7.041E-04} \\
\midrule
 \multirow{4}*{NBVS} & BBSE-soft &           2.116E-02 &           8.849E-03 &           6.384E-03 &           3.396E-03 &           2.250E-03 &           1.120E-03 \\
   & RLLS-soft &           1.938E-02 &           8.762E-03 &           6.149E-03 &           3.373E-03 &           2.164E-03 &           1.115E-03 \\
   & MLLS &           6.038E-03 &  \textbf{5.722E-03} &           2.791E-03 &           2.310E-03 &           1.372E-03 &           8.431E-04 \\
   & ELSA &           3.075E-03 &           5.775E-03 &  \textbf{8.751E-04} &           2.116E-03 &  \textbf{3.657E-04} &  \textbf{7.004E-04} \\
\midrule
 \multirow{4}*{TS} & BBSE-soft &           2.060E-02 &           8.738E-03 &           6.385E-03 &           3.459E-03 &           2.336E-03 &           1.191E-03 \\
   & RLLS-soft &           1.833E-02 &           8.617E-03 &           5.906E-03 &           3.429E-03 &           2.263E-03 &           1.188E-03 \\
   & MLLS &           3.491E-02 &           9.200E-03 &           2.462E-02 &           5.252E-03 &           1.833E-02 &           2.684E-03 \\
   & ELSA &  \textbf{2.737E-03} &           5.742E-03 &  \textbf{8.904E-04} &  \textbf{2.101E-03} &           3.819E-04 &           7.167E-04 \\
\midrule
 \multirow{4}*{VS} & BBSE-soft &           2.092E-02 &           9.043E-03 &           6.375E-03 &           3.410E-03 &           2.252E-03 &           1.110E-03 \\
   & RLLS-soft &           1.910E-02 &           8.935E-03 &           6.028E-03 &           3.384E-03 &           2.162E-03 &           1.104E-03 \\
   & MLLS &           4.222E-03 &           5.835E-03 &           1.082E-03 &           2.205E-03 &           4.899E-04 &           7.624E-04 \\
   & ELSA &           3.873E-03 &           5.935E-03 &           9.983E-04 &  \textbf{2.089E-03} &  \textbf{3.664E-04} &  \textbf{7.050E-04} \\
\bottomrule
\end{tabular}
\end{table}

\begin{table}[ht]
    \centering
        \caption{Prediction improvement comparison on CIFAR-10 with Dirichlet shift. The value is the mean delta accuracy trimmed $5\%$ extreme values. The bold values under the same setting (column) are the top-3 results.} 
    \label{table.CIFAR-10_Dirichlet_acc}

\begin{tabular}{llllllll}
\toprule
 & Settings & (500, 0.1) & (500, 1.0) & (1500, 0.1) & (1500, 1.0) & (4500, 0.1) & (4500, 1.0) \\
\midrule
\multirow{6}{*}{None} & BBSE-hard & 6.182e-02 & 1.722e-02 & 6.520e-02 & 1.926e-02 & 6.467e-02 & 1.881e-02 \\
 & RLLS-hard & 6.151e-02 & 1.711e-02 & 6.525e-02 & 1.926e-02 & 6.470e-02 & 1.880e-02 \\
 & BBSE-soft & 6.293e-02 & 1.729e-02 & 6.580e-02 & 1.946e-02 & 6.514e-02 & 1.867e-02 \\
 & RLLS-soft & 6.296e-02 & 1.727e-02 & 6.589e-02 & 1.947e-02 & 6.523e-02 & 1.867e-02 \\
 & MLLS & 6.538e-02 & 1.811e-02 & 6.689e-02 & 1.964e-02 & 6.551e-02 & 1.859e-02 \\
 & ELSA & 6.787e-02 & 1.844e-02 & 6.900e-02 & 1.981e-02 & 6.756e-02 & 1.906e-02 \\
\cline{1-8}
\multirow{4}{*}{BCTS} & BBSE-soft & 6.324e-02 & 2.080e-02 & 6.828e-02 & 2.472e-02 & 6.873e-02 & 2.480e-02 \\
 & RLLS-soft & 6.313e-02 & 2.078e-02 & 6.830e-02 & 2.473e-02 & 6.871e-02 & 2.480e-02 \\
 & MLLS & 6.729e-02 & \textbf{2.169e-02} & 6.988e-02 & \textbf{2.496e-02} & 6.962e-02 & 2.503e-02 \\
 & ELSA & \textbf{6.804e-02} & \textbf{2.196e-02} & \textbf{7.050e-02} & \textbf{2.512e-02} & \textbf{6.981e-02} & \textbf{2.511e-02} \\
\cline{1-8}
\multirow{4}{*}{NBVS} & BBSE-soft & 6.287e-02 & 2.024e-02 & 6.876e-02 & 2.452e-02 & 6.847e-02 & 2.423e-02 \\
 & BBSE-soft & 6.269e-02 & 2.027e-02 & 6.884e-02 & 2.452e-02 & 6.847e-02 & 2.424e-02 \\
 & MLLS & 6.682e-02 & 2.124e-02 & \textbf{7.044e-02} & 2.478e-02 & 6.940e-02 & 2.428e-02 \\
 & ELSA & \textbf{6.802e-02} & \textbf{2.162e-02} & \textbf{7.097e-02} & \textbf{2.498e-02} & \textbf{6.976e-02} & 2.439e-02 \\
\cline{1-8}
\multirow{4}{*}{TS} & BBSE-soft & 6.458e-02 & 1.833e-02 & 6.781e-02 & 2.016e-02 & 6.721e-02 & 1.992e-02 \\
 & RLLS-soft & 6.447e-02 & 1.831e-02 & 6.779e-02 & 2.019e-02 & 6.726e-02 & 1.991e-02 \\
 & MLLS & 6.798e-02 & 1.891e-02 & 6.947e-02 & 2.044e-02 & 6.815e-02 & 2.013e-02 \\
 & ELSA & \textbf{6.869e-02} & 1.871e-02 & 7.013e-02 & 2.071e-02 & 6.869e-02 & 2.023e-02 \\
\cline{1-8}
\multirow{4}{*}{VS} & BBSE-soft & 6.053e-02 & 1.909e-02 & 6.749e-02 & 2.447e-02 & 6.841e-02 & 2.490e-02 \\
 & RLLS-soft & 6.036e-02 & 1.909e-02 & 6.750e-02 & 2.445e-02 & 6.840e-02 & 2.490e-02 \\
 & MLLS & 6.467e-02 & 2.004e-02 & 6.930e-02 & 2.475e-02 & 6.960e-02 & \textbf{2.520e-02} \\
 & ELSA & 6.438e-02 & 2.011e-02 & 6.997e-02 & 2.475e-02 & \textbf{6.983e-02} & \textbf{2.516e-02} \\
\cline{1-8}
\bottomrule
\end{tabular}

\end{table}

\begin{table}[ht]
    \centering
        \caption{Method comparison on CIFAR-100 with Dirichlet shift. The value is the mean MSE trimmed $5\%$ extreme values. The bold values under the same setting (column) are the top-3 results.}
    \label{table.CIFAR-100_Dirichlet}
    
\begin{tabular}{llllll}
\toprule
 \multirow{2}*{Calibration}  & \multirow{2}*{Adaptation}  & \multicolumn{4}{c}{Setting: (Sample Size $n$, Shift Parameter $\alpha$)} \\
  \cline{3-6}
   &  &     (1500, 0.1) &     (1500, 1.0) &     (4500, 0.1) &     (4500, 1.0) \\
\midrule
 \multirow{6}*{None} & BBSE-hard &           8.969 &           4.304 &           1.871 &           0.687 \\
   & RLLS-hard &           5.284 &           2.072 &           0.945 &           0.405 \\
   & BBSE-soft &           4.595 &           2.637 &           1.413 &           0.488 \\
   & RLLS-soft &           4.429 &           1.563 &           0.761 &           0.309 \\
   & MLLS &           3.451 &           1.504 &           2.020 &           0.737 \\
   & ELSA &           1.086 &           2.328 &           0.288 &           0.207 \\
\midrule
 \multirow{4}*{BCTS}  & BBSE-soft &           3.955 &           2.353 &           1.085 &           0.365 \\
   & RLLS-soft &           5.130 &           1.618 &           0.683 &           0.262 \\
   & MLLS &  \textbf{0.430} &  \textbf{0.255} &           0.220 &  \textbf{0.133} \\
   & ELSA &  \textbf{0.686} &           0.861 &  \textbf{0.125} &           0.141 \\
\midrule
 \multirow{4}*{NBVS} & BBSE-soft &           3.772 &           2.451 &           1.093 &           0.360 \\
   & RLLS-soft &           5.484 &           1.583 &           0.692 &           0.264 \\
   & MLLS &           0.722 &  \textbf{0.387} &           0.269 &           0.158 \\
   & ELSA &           0.768 &           0.812 &           0.133 &           0.142 \\
\midrule
 \multirow{4}*{TS} & BBSE-soft &           3.700 &           2.500 &           1.107 &           0.391 \\
   & RLLS-soft &           6.193 &           1.847 &           0.782 &           0.278 \\
   & MLLS &           1.758 &           0.949 &           0.806 &           0.396 \\
   & ELSA &           0.920 &           0.683 &  \textbf{0.127} &           0.136 \\
\midrule
 \multirow{4}*{VS}  & BBSE-soft &           3.808 &           2.321 &           1.078 &           0.357 \\
   & RLLS-soft &           5.749 &           1.574 &           0.716 &           0.268 \\
   & MLLS &  \textbf{0.500} &  \textbf{0.334} &           0.190 &  \textbf{0.134} \\
   & ELSA &           0.958 &           0.598 &  \textbf{0.131} &  \textbf{0.134} \\
\bottomrule
\end{tabular}
\end{table}

\begin{table}[ht]
    \centering
        \caption{Prediction improvement comparison on CIFAR-100 with Dirichlet shift. The value is the mean delta accuracy trimmed $5\%$ extreme values. The bold values under the same setting (column) are the top-3 results.} 
    \label{table.CIFAR-100_Dirichlet_acc}
\begin{tabular}{llllll}
\toprule
 & Settings & (1500, 0.1) & (1500, 1.0) & (4500, 0.1) & (4500, 1.0) \\
\midrule
\multirow{6}{*}{None} & BBSE-hard & 0.145 & 0.119 & 0.161 & 0.142 \\
 & RLLS-hard & 0.136 & 0.120 & 0.171 & 0.153 \\
 & BBSE-soft & 0.151 & 0.128 & 0.164 & 0.148 \\
 & RLLS-soft & 0.141 & 0.127 & 0.177 & 0.160 \\
 & MLLS & 0.138 & 0.119 & 0.140 & 0.122 \\
 & ELSA & 0.195 & 0.152 & 0.213 & 0.173 \\
\cline{1-6}
\multirow{4}{*}{BCTS} & BBSE-soft & 0.229 & 0.185 & 0.249 & 0.208 \\
 & RLLS-soft & 0.214 & 0.182 & 0.255 & 0.214 \\
 & MLLS & 0.247 & 0.204 & 0.257 & 0.216 \\
 & ELSA & \textbf{0.263} & \textbf{0.205} & \textbf{0.272} & \textbf{0.222} \\
\cline{1-6}
\multirow{4}{*}{NBVS} & BBSE-soft & 0.227 & 0.182 & 0.246 & 0.207 \\
 & RLLS-soft & 0.211 & 0.179 & 0.253 & 0.212 \\
 & MLLS & 0.245 & 0.200 & 0.255 & 0.213 \\
 & ELSA & \textbf{0.261} & 0.204 & \textbf{0.270} & \textbf{0.220} \\
\cline{1-6}
\multirow{4}{*}{TS} & BBSE-soft & 0.227 & 0.184 & 0.243 & 0.205 \\
 & RLLS-soft & 0.208 & 0.179 & 0.249 & 0.210 \\
 & MLLS & 0.251 & \textbf{0.207} & 0.255 & 0.214 \\
 & ELSA & \textbf{0.259} & \textbf{0.204} & 0.268 & 0.219 \\
\cline{1-6}
\multirow{4}{*}{VS} & BBSE-soft & 0.221 & 0.175 & 0.249 & 0.207 \\
 & RLLS-soft & 0.205 & 0.172 & 0.254 & 0.212 \\
 & MLLS & 0.239 & 0.194 & 0.261 & 0.216 \\
 & ELSA & 0.256 & 0.197 & \textbf{0.272} & \textbf{0.222} \\
\cline{1-6}
\bottomrule
\end{tabular}
\end{table}

\begin{table}[ht]
    \centering
            \caption{Method comparison on MNIST with tweakone shift. The value is the mean MSE trimmed $5\%$ extreme values. The bold values under the same setting (column) are the top-3 results.} 
    \label{table.MINIST_tweakone}
\begin{tabular}{llllllll}
\toprule
 \multirow{2}*{Calibration}  & \multirow{2}*{Adaptation}  & \multicolumn{6}{c}{Setting: (Sample Size $n$, Shift Parameter $\rho$)} \\
  \cline{3-8}
   &  &         (500, 0.01) &          (500, 0.9) &        (1500, 0.01) &         (1500, 0.9) &        (4500, 0.01) &         (4500, 0.9) \\
\midrule
 \multirow{6}*{None} & BBSE-hard &           9.903E-03 &           2.468E-02 &           2.911E-03 &           1.341E-02 &           1.155E-03 &           7.261E-03 \\
   & RLLS-hard &           9.893E-03 &           2.164E-02 &           2.910E-03 &           1.237E-02 &           1.155E-03 &           6.991E-03 \\
   & BBSE-soft &           7.735E-03 &           2.297E-02 &           2.256E-03 &           9.890E-03 &           9.310E-04 &           4.213E-03 \\
   & RLLS-soft &           7.715E-03 &           2.065E-02 &           2.256E-03 &           9.288E-03 &           9.310E-04 &           4.161E-03 \\
   & MLLS &           8.275E-03 &           2.001E-02 &           2.622E-03 &           9.132E-03 &           1.046E-03 &           5.777E-03 \\
   & ELSA &           7.128E-03 &           4.157E-03 &           2.116E-03 &  \textbf{1.707E-03} &           8.722E-04 &           8.473E-04 \\
   \midrule
 \multirow{4}*{BCTS} & BBSE-soft &           7.415E-03 &           2.363E-02 &           2.217E-03 &           9.650E-03 &           8.864E-04 &           4.448E-03 \\
   & RLLS-soft &           7.396E-03 &           2.137E-02 &           2.217E-03 &           9.063E-03 &           8.864E-04 &           4.396E-03 \\
   & MLLS &  \textbf{6.757E-03} &  \textbf{1.991E-03} &  \textbf{2.105E-03} &  \textbf{7.603E-04} &  \textbf{8.237E-04} &  \textbf{3.925E-04} \\
   & ELSA &  \textbf{6.932E-03} &  \textbf{3.447E-03} &  \textbf{2.110E-03} &           1.792E-03 &           8.579E-04 &           8.847E-04 \\
   \midrule
 \multirow{4}*{NBVS} & BBSE-soft &           7.723E-03 &           2.343E-02 &           2.271E-03 &           9.859E-03 &           9.128E-04 &           4.477E-03 \\
   & RLLS-soft &           7.717E-03 &           2.137E-02 &           2.271E-03 &           9.294E-03 &           9.128E-04 &           4.421E-03 \\
   & MLLS &           7.138E-03 &           5.809E-03 &           2.146E-03 &           1.784E-03 &  \textbf{8.338E-04} &  \textbf{6.836E-04} \\
   & ELSA &           7.171E-03 &           3.554E-03 &           2.145E-03 &           1.863E-03 &           8.793E-04 &           8.864E-04 \\
   \midrule
 \multirow{4}*{TS}  & BBSE-soft &           7.380E-03 &           2.201E-02 &           2.173E-03 &           9.866E-03 &           8.838E-04 &           4.553E-03 \\
   & RLLS-soft &           7.368E-03 &           1.988E-02 &           2.173E-03 &           9.262E-03 &           8.838E-04 &           4.498E-03 \\
   & MLLS &           7.253E-03 &           1.600E-02 &           2.223E-03 &           5.479E-03 &           8.633E-04 &           2.559E-03 \\
   & ELSA &  \textbf{6.828E-03} &           3.612E-03 &  \textbf{2.065E-03} &           1.793E-03 &           8.562E-04 &           8.844E-04 \\
   \midrule
 \multirow{4}*{VS}   & BBSE-soft &           7.460E-03 &           2.386E-02 &           2.239E-03 &           9.917E-03 &           9.179E-04 &           4.716E-03 \\
   & RLLS-soft &           7.457E-03 &           2.152E-02 &           2.239E-03 &           9.301E-03 &           9.179E-04 &           4.650E-03 \\
   & MLLS &           6.947E-03 &  \textbf{2.943E-03} &           2.129E-03 &  \textbf{8.531E-04} &  \textbf{8.269E-04} &  \textbf{4.348E-04} \\
   & ELSA &           7.030E-03 &           3.598E-03 &           2.123E-03 &           1.762E-03 &           8.853E-04 &           9.053E-04 \\
\bottomrule
\end{tabular}
\end{table}

\begin{table}[ht]
    \centering
            \caption{Prediction improvement comparison on MNIST with tweakone shift. The value is the mean delta accuracy trimmed $5\%$ extreme values. The bold values under the same setting (column) are the top-3 results.} 
    \label{table.MINIST_tweakone_acc}
\begin{tabular}{llllllll}
\toprule
 & Settings & (500, 0.01) & (500, 0.9) & (1500, 0.01) & (1500, 0.9) & (4500, 0.01) & (4500, 0.9) \\
\midrule
\multirow{6}{*}{None} & BBSE-hard & 1.311e-03 & 4.993e-02 & 2.474e-03 & 5.033e-02 & 3.262e-03 & 5.317e-02 \\
 & RLLS-hard & 1.311e-03 & 4.936e-02 & 2.474e-03 & 4.995e-02 & 3.262e-03 & 5.302e-02 \\
 & BBSE-soft & 1.267e-03 & 5.218e-02 & 2.548e-03 & 5.438e-02 & 3.281e-03 & 5.625e-02 \\
 & RLLS-soft & 1.267e-03 & 5.162e-02 & 2.548e-03 & 5.399e-02 & 3.281e-03 & 5.618e-02 \\
 & MLLS & 1.022e-03 & 5.669e-02 & 1.815e-03 & 5.667e-02 & 2.632e-03 & 5.637e-02 \\
 & ELSA & 1.267e-03 & 5.867e-02 & 2.422e-03 & 5.881e-02 & 3.195e-03 & 6.018e-02 \\
\cline{1-8}
\multirow{4}{*}{BCTS} & BBSE-soft & -6.667e-04 & 5.253e-02 & 2.800e-03 & 5.759e-02 & 3.783e-03 & 6.102e-02 \\
 & RLLS-soft & -6.444e-04 & 5.229e-02 & 2.800e-03 & 5.744e-02 & 3.783e-03 & 6.089e-02 \\
 & MLLS & 8.889e-05 & \textbf{6.233e-02} & 2.785e-03 & \textbf{6.379e-02} & 3.800e-03 & 6.472e-02 \\
 & ELSA & 1.111e-04 & 6.124e-02 & 2.770e-03 & 6.262e-02 & 3.756e-03 & 6.442e-02 \\
\cline{1-8}
\multirow{4}{*}{NBVS} & BBSE-soft & -2.444e-04 & 5.198e-02 & 2.719e-03 & 5.716e-02 & 3.775e-03 & 6.098e-02 \\
 & RLLS-soft & -2.444e-04 & 5.162e-02 & 2.719e-03 & 5.690e-02 & 3.775e-03 & 6.085e-02 \\
 & MLLS & 5.778e-04 & 6.167e-02 & 2.689e-03 & 6.338e-02 & 3.689e-03 & 6.445e-02 \\
 & ELSA & 5.111e-04 & 6.064e-02 & 2.726e-03 & 6.242e-02 & 3.726e-03 & 6.440e-02 \\
\cline{1-8}
\multirow{4}{*}{TS} & BBSE-soft & \textbf{2.244e-03} & 5.387e-02 & \textbf{3.311e-03} & 5.792e-02 & \textbf{3.921e-03} & 6.159e-02 \\
 & RLLS-soft & \textbf{2.244e-03} & 5.331e-02 & \textbf{3.311e-03} & 5.767e-02 & \textbf{3.921e-03} & 6.146e-02 \\
 & MLLS & \textbf{2.711e-03} & \textbf{6.387e-02} & 3.296e-03 & \textbf{6.524e-02} & \textbf{3.943e-03} & \textbf{6.549e-02} \\
 & ELSA & \textbf{2.533e-03} & \textbf{6.278e-02} & \textbf{3.311e-03} & 6.337e-02 & \textbf{3.921e-03} & \textbf{6.518e-02} \\
\cline{1-8}
\multirow{4}{*}{VS} & BBSE-soft & -2.400e-03 & 5.087e-02 & 2.126e-03 & 5.724e-02 & 3.832e-03 & 6.062e-02 \\
 & RLLS-soft & -2.400e-03 & 5.067e-02 & 2.126e-03 & 5.693e-02 & 3.832e-03 & 6.046e-02 \\
 & MLLS & -2.156e-03 & 6.071e-02 & 2.141e-03 & \textbf{6.350e-02} & 3.805e-03 & \textbf{6.482e-02} \\
 & ELSA & -2.267e-03 & 6.009e-02 & 2.207e-03 & 6.250e-02 & 3.844e-03 & 6.471e-02 \\
\cline{1-8}
\bottomrule
\end{tabular}
\end{table}

\begin{table}[ht]
    \centering
            \caption{Method comparison on CIFAR-10 with tweakone shift. The value is the mean MSE trimmed $5\%$ extreme values. The bold values under the same setting (column) are the top-3 results.}
    \label{table.CIFAR-10_tweakone}
\begin{tabular}{llllllll}
\toprule
 \multirow{2}*{Calibration}  & \multirow{2}*{Adaptation}  & \multicolumn{6}{c}{Setting: (Sample Size $n$, Shift Parameter $\rho$)} \\
  \cline{3-8}
   &  &         (500, 0.01) &          (500, 0.9) &        (1500, 0.01) &         (1500, 0.9) &        (4500, 0.01) &         (4500, 0.9) \\
\midrule
 \multirow{6}*{None} & BBSE-hard &           8.587E-03 &           7.757E-02 &           2.980E-03 &           2.577E-02 &           9.673E-04 &           9.601E-03 \\
   & RLLS-hard &           8.277E-03 &           6.795E-02 &           2.962E-03 &           2.387E-02 &           9.673E-04 &           9.333E-03 \\
   & BBSE-soft &           6.433E-03 &           6.309E-02 &           2.216E-03 &           2.024E-02 &           6.637E-04 &           8.716E-03 \\
   & RLLS-soft &           6.268E-03 &           5.800E-02 &           2.205E-03 &           1.961E-02 &           6.637E-04 &           8.683E-03 \\
   & MLLS &           6.213E-03 &           1.171E-01 &           2.351E-03 &           7.724E-02 &           9.129E-04 &           6.918E-02 \\
   & ELSA &           5.729E-03 &  \textbf{3.817E-03} &           1.990E-03 &           1.562E-03 &           6.161E-04 &           5.878E-04 \\
   \midrule
 \multirow{4}*{BCTS} & BBSE-soft &           6.349E-03 &           6.278E-02 &           2.172E-03 &           1.827E-02 &           6.749E-04 &           7.112E-03 \\
   & RLLS-soft &           6.223E-03 &           5.694E-02 &           2.163E-03 &           1.769E-02 &           6.749E-04 &           7.037E-03 \\
   & MLLS &  \textbf{5.242E-03} &           6.489E-03 &  \textbf{1.788E-03} &           2.392E-03 &  \textbf{5.669E-04} &           1.169E-03 \\
   & ELSA &           5.648E-03 &  \textbf{4.147E-03} &           1.942E-03 &  \textbf{1.391E-03} &           6.065E-04 &  \textbf{5.022E-04} \\
   \midrule
 \multirow{4}*{NBVS} & BBSE-soft &           6.465E-03 &           6.713E-02 &           2.170E-03 &           1.801E-02 &           6.816E-04 &           7.088E-03 \\
   & RLLS-soft &           6.334E-03 &           6.172E-02 &           2.161E-03 &           1.760E-02 &           6.816E-04 &           7.031E-03 \\
   & MLLS &  \textbf{5.469E-03} &           1.333E-02 &  \textbf{1.857E-03} &           4.281E-03 &  \textbf{5.731E-04} &           2.325E-03 \\
   & ELSA &           5.877E-03 &           4.356E-03 &           1.968E-03 &           1.604E-03 &           6.125E-04 &           5.496E-04 \\
   \midrule
 \multirow{4}*{TS} & BBSE-soft &           6.457E-03 &           6.357E-02 &           2.172E-03 &           2.078E-02 &           6.603E-04 &           9.362E-03 \\
   & RLLS-soft &           6.282E-03 &           5.885E-02 &           2.163E-03 &           2.023E-02 &           6.603E-04 &           9.342E-03 \\
   & MLLS &           6.445E-03 &           1.457E-01 &           2.400E-03 &           1.034E-01 &           9.904E-04 &           9.584E-02 \\
   & ELSA &           5.674E-03 &  \textbf{3.963E-03} &           1.955E-03 &  \textbf{1.303E-03} &           6.105E-04 &  \textbf{5.115E-04} \\
   \midrule
 \multirow{4}*{VS}  & BBSE-soft &           6.436E-03 &           6.561E-02 &           2.173E-03 &           1.822E-02 &           6.885E-04 &           6.906E-03 \\
   & RLLS-soft &           6.308E-03 &           5.929E-02 &           2.164E-03 &           1.772E-02 &           6.885E-04 &           6.833E-03 \\
   & MLLS &  \textbf{5.438E-03} &           9.515E-03 &  \textbf{1.832E-03} &           2.615E-03 &  \textbf{5.643E-04} &           1.232E-03 \\
   & ELSA &           5.717E-03 &           4.904E-03 &           1.953E-03 &  \textbf{1.425E-03} &           6.157E-04 &  \textbf{5.214E-04} \\
\bottomrule
\end{tabular}
\end{table}

\begin{table}[ht]
    \centering
            \caption{Prediction improvement comparison on CIFAR-10 with tweakone shift. The value is the mean delta accuracy trimmed $5\%$ extreme values. The bold values under the same setting (column) are the top-3 results.}
    \label{table.CIFAR-10_tweakone_acc}
\begin{tabular}{llllllll}
\toprule
 & Settings & (500, 0.01) & (500, 0.9) & (1500, 0.01) & (1500, 0.9) & (4500, 0.01) & (4500, 0.9) \\
\midrule
\multirow{6}{*}{None} & BBSE-hard & 5.489e-03 & 1.480e-01 & 7.630e-03 & 1.580e-01 & 7.793e-03 & 1.650e-01 \\
 & RLLS-hard & 5.511e-03 & 1.479e-01 & 7.630e-03 & 1.580e-01 & 7.793e-03 & 1.650e-01 \\
 & BBSE-soft & 5.422e-03 & 1.526e-01 & 7.615e-03 & 1.589e-01 & 7.886e-03 & 1.648e-01 \\
 & RLLS-soft & 5.467e-03 & 1.526e-01 & 7.615e-03 & 1.590e-01 & 7.886e-03 & 1.648e-01 \\
 & MLLS & 6.000e-03 & 1.627e-01 & 7.630e-03 & 1.617e-01 & 7.980e-03 & 1.633e-01 \\
 & ELSA & 6.244e-03 & 1.676e-01 & 7.844e-03 & 1.676e-01 & 8.005e-03 & 1.690e-01 \\
\cline{1-8}
\multirow{4}{*}{BCTS} & BBSE-soft & 7.733e-03 & 1.599e-01 & 1.156e-02 & 1.701e-01 & 1.232e-02 & 1.772e-01 \\
 & RLLS-soft & 7.733e-03 & 1.598e-01 & \textbf{1.157e-02} & 1.700e-01 & 1.232e-02 & 1.772e-01 \\
 & MLLS & \textbf{8.533e-03} & \textbf{1.728e-01} & \textbf{1.181e-02} & 1.757e-01 & 1.230e-02 & 1.782e-01 \\
 & ELSA & \textbf{8.467e-03} & \textbf{1.739e-01} & \textbf{1.184e-02} & \textbf{1.773e-01} & 1.236e-02 & \textbf{1.794e-01} \\
\cline{1-8}
\multirow{4}{*}{NBVS} & BBSE-soft & 6.956e-03 & 1.559e-01 & 1.108e-02 & 1.684e-01 & 1.166e-02 & 1.760e-01 \\
 & RLLS-soft & 6.956e-03 & 1.559e-01 & 1.108e-02 & 1.683e-01 & 1.166e-02 & 1.759e-01 \\
 & MLLS & 7.600e-03 & 1.700e-01 & 1.122e-02 & 1.744e-01 & 1.168e-02 & 1.766e-01 \\
 & ELSA & 7.733e-03 & \textbf{1.728e-01} & 1.113e-02 & \textbf{1.764e-01} & 1.171e-02 & \textbf{1.790e-01} \\
\cline{1-8}
\multirow{4}{*}{TS} & BBSE-soft & 5.511e-03 & 1.592e-01 & 7.719e-03 & 1.656e-01 & 7.926e-03 & 1.722e-01 \\
 & RLLS-soft & 5.533e-03 & 1.593e-01 & 7.719e-03 & 1.655e-01 & 7.926e-03 & 1.722e-01 \\
 & MLLS & 6.156e-03 & 1.712e-01 & 7.822e-03 & 1.706e-01 & 7.884e-03 & 1.726e-01 \\
 & ELSA & 6.111e-03 & \textbf{1.730e-01} & 7.881e-03 & 1.732e-01 & 7.958e-03 & 1.749e-01 \\
\cline{1-8}
\multirow{4}{*}{VS} & BBSE-soft & 7.067e-03 & 1.551e-01 & 1.124e-02 & 1.689e-01 & \textbf{1.269e-02} & 1.772e-01 \\
 & RLLS-soft & 7.067e-03 & 1.551e-01 & 1.124e-02 & 1.689e-01 & \textbf{1.269e-02} & 1.772e-01 \\
 & MLLS & 7.600e-03 & 1.684e-01 & 1.139e-02 & 1.754e-01 & \textbf{1.275e-02} & 1.788e-01 \\
 & ELSA & \textbf{7.844e-03} & 1.711e-01 & 1.144e-02 & \textbf{1.767e-01} & \textbf{1.270e-02} & \textbf{1.799e-01} \\
\cline{1-8}
\bottomrule
\end{tabular}
\end{table}

\begin{table}[ht!]
    \centering
       \caption{Method comparison on CIFAR-100 with tweakone shift. The value is the mean MSE trimmed $5\%$ extreme values. The bold values under the same setting (column) are the top-3 results.}
    \label{table.CIFAR-100_tweakone}
    \begin{tabular}{llllll}
\toprule
 \multirow{2}*{Calibration}  & \multirow{2}*{Adaptation}  & \multicolumn{4}{c}{Setting: (Sample Size $n$, Shift Parameter $\rho$)} \\
  \cline{3-6}
   &  &    (1500, 0.01) &     (1500, 0.9) &    (4500, 0.01) &     (4500, 0.9) \\
\midrule
 \multirow{6}*{None} & BBSE-hard &           2.755 &          29.304 &           0.390 &           5.229 \\
   & RLLS-hard &           1.246 &          28.798 &           0.230 &           3.942 \\
   & BBSE-soft &           1.750 &          14.294 &           0.286 &           3.399 \\
   & RLLS-soft &           0.944 &          28.426 &           0.183 &           3.262 \\
   & MLLS &           0.966 &           7.295 &           0.480 &           3.822 \\
   & ELSA &           1.182 &           5.640 &           0.130 &           1.068 \\
   \midrule
 \multirow{4}*{BCTS} & BBSE-soft &           1.864 &          12.403 &           0.225 &           2.638 \\
   & RLLS-soft &           0.937 &          40.160 &           0.159 &           3.207 \\
   & MLLS &  \textbf{0.169} &  \textbf{1.725} &  \textbf{0.079} &           1.099 \\
   & ELSA &           0.408 &           2.698 &           0.089 &  \textbf{0.400} \\
   \midrule
 \multirow{4}*{NBVS} & BBSE-soft &           2.074 &          12.004 &           0.223 &           2.571 \\
   & RLLS-soft &           0.981 &          37.647 &           0.159 &           3.072 \\
   & MLLS &  \textbf{0.222} &  \textbf{2.239} &           0.091 &           1.024 \\
   & ELSA &           0.548 &           2.562 &           0.091 &           0.453 \\
   \midrule
 \multirow{4}*{TS} & BBSE-soft &           1.788 &          13.292 &           0.232 &           2.765 \\
   & RLLS-soft &           0.988 &          49.960 &           0.169 &           4.078 \\
   & MLLS &           0.578 &           4.378 &           0.204 &           2.438 \\
   & ELSA &           0.412 &  \textbf{2.207} &           0.086 &  \textbf{0.409} \\
   \midrule
 \multirow{4}*{VS}  & BBSE-soft &           1.654 &          10.533 &           0.218 &           2.543 \\
   & RLLS-soft &           0.961 &          36.722 &           0.158 &           3.355 \\
   & MLLS &  \textbf{0.192} &           2.333 &  \textbf{0.073} &           1.155 \\
   & ELSA &           0.631 &           2.669 &  \textbf{0.086} &  \textbf{0.430} \\
\bottomrule
\end{tabular}
\end{table}

\begin{table}[ht]
    \centering
            \caption{Prediction improvement comparison on CIFAR-100 with tweakone shift. The value is the mean delta accuracy trimmed $5\%$ extreme values. The bold values under the same setting (column) are the top-3 results.}
    \label{table.CIFAR-100_tweakone_acc}
\begin{tabular}{llllll}
\toprule
 & Settings & (1500, 0.01) & (1500, 0.9) & (4500, 0.01) & (4500, 0.9) \\
\midrule
\multirow{6}{*}{None} & BBSE-hard & 0.115 & 0.208 & 0.135 & 0.230 \\
 & RLLS-hard & 0.116 & 0.182 & 0.147 & 0.244 \\
 & BBSE-soft & 0.122 & 0.223 & 0.140 & 0.245 \\
 & RLLS-soft & 0.122 & 0.186 & 0.152 & 0.256 \\
 & MLLS & 0.108 & 0.214 & 0.111 & 0.219 \\
 & ELSA & 0.147 & 0.303 & 0.160 & 0.322 \\
\cline{1-6}
\multirow{4}{*}{BCTS} & BBSE-soft & 0.172 & 0.355 & 0.194 & 0.393 \\
 & RLLS-soft & 0.171 & 0.285 & 0.199 & 0.398 \\
 & MLLS & 0.190 & 0.413 & 0.201 & 0.427 \\
 & ELSA & \textbf{0.193} & \textbf{0.432} & \textbf{0.207} & \textbf{0.452} \\
\cline{1-6}
\multirow{4}{*}{NBVS} & BBSE-soft & 0.171 & 0.345 & 0.194 & 0.381 \\
 & RLLS-soft & 0.168 & 0.275 & 0.199 & 0.386 \\
 & MLLS & 0.188 & 0.395 & 0.200 & 0.414 \\
 & ELSA & \textbf{0.192} & \textbf{0.432} & \textbf{0.206} & \textbf{0.452} \\
\cline{1-6}
\multirow{4}{*}{TS} & BBSE-soft & 0.173 & 0.356 & 0.192 & 0.387 \\
 & RLLS-soft & 0.172 & 0.254 & 0.196 & 0.384 \\
 & MLLS & 0.192 & 0.412 & 0.199 & 0.418 \\
 & ELSA & \textbf{0.193} & 0.429 & 0.204 & 0.447 \\
\cline{1-6}
\multirow{4}{*}{VS} & BBSE-soft & 0.165 & 0.345 & 0.194 & 0.393 \\
 & RLLS-soft & 0.163 & 0.276 & 0.198 & 0.395 \\
 & MLLS & 0.183 & 0.397 & 0.202 & 0.427 \\
 & ELSA & 0.186 & \textbf{0.430} & \textbf{0.207} & \textbf{0.456} \\
\cline{1-6}
\bottomrule
\end{tabular}
\end{table}

\section{Motivations for $\mathbf{h}_{\mathrm{ELSA}}(\x)$}
\label{appendx-motivation}
The motivation of the $h_{\mathrm{ELSA}}(\boldsymbol{x})$ function starts from the score function with respect to $\boldsymbol{\omega}^{(-1)}=(\omega_1,\dots,\omega_{k-1})$, and denoted by $\mathbf{S}_{\boldsymbol{\omega}}(\boldsymbol{x})$. The $i$-th element of the score function is given by
$$
[\mathbf{S}_{\boldsymbol{\omega}}(\boldsymbol{x})]_i\propto\frac{p_s(\boldsymbol{x})}{p_t(\boldsymbol{x})}\left\{p_s(y=i|\boldsymbol{x})-p_s(y=k|\boldsymbol{x})\right\},\quad i=1,\dots,k-1.
$$
We could use $\mathbf{S}_{\boldsymbol{\omega}}(\boldsymbol{x})$ directly to construct an influence function for a RAL estimator. But we can improve efficiency (i.e., reducing estimation error) by projecting it to the perpendicular space $\Lambda^{\perp}$. Prioritizing computational efficiency and feasibility, we approximate the projection $\Pi(\mathbf{S}_{\boldsymbol{\omega}}(\boldsymbol{x})|\Lambda^\perp)$ with
$$
\Pi(S_i(\boldsymbol{x})\mid \Lambda^\perp)\propto \kappa(\boldsymbol{x})S_i(\boldsymbol{x}),
$$
where $\kappa(\boldsymbol{x})$ is a "bridging" function that needs to satisfy
$$
\frac{1-\kappa(\boldsymbol{x})}{\kappa(\boldsymbol{x})}=E_t\left\{\frac{1-\Pr(R=1|Y,\boldsymbol{X})}{\Pr(R=1|Y,\boldsymbol{X})}|\boldsymbol{x}\right\}.
$$
Under the label shift assumption, we further have
$$
\frac{1-\kappa(\boldsymbol{x})}{\kappa(\boldsymbol{x})}=\frac{1-\pi}{\pi}E_t\left\{\frac{p_t(Y)}{p_s(Y)}|\boldsymbol{x}\right\}.
$$

Next we will show tht the proposed function $h_{\mathrm{ELSA}}(\boldsymbol{x})$ is proportional to $\kappa(\boldsymbol{x})\mathbf{S}_i(\boldsymbol{x})$. Because $\kappa(\boldsymbol{x})\mathbf{S}_i(\boldsymbol{x})=\kappa(\boldsymbol{x})\dfrac{p_s(\boldsymbol{x})}{p_t(\boldsymbol{\boldsymbol{x}})}\left\{p_s(y=i|\boldsymbol{x})-p_s(y=k|\boldsymbol{x})\right\}$, we only need to verify that the denominator of $h_{\mathrm{ELSA}}(\boldsymbol{x})$ is proportional to the reciprocal $\kappa(\boldsymbol{x})\dfrac{p_s(\boldsymbol{x})}{p_t(\boldsymbol{x})}$: the denominator of $h_{\mathrm{ELSA}}(\boldsymbol{x})$ is
$$
\begin{aligned}
&\frac{E_s(\rho^2\mid \boldsymbol{x})}{\pi} + \frac{E_s(\rho\mid \boldsymbol{x})}{1-\pi}\\
\propto& \frac{p_t(\boldsymbol{x})}{p_s(\boldsymbol{x})}\frac{1-\kappa(\boldsymbol{x})}{\kappa(\boldsymbol{x})}\frac1{1-\pi} + \frac{p_t(\boldsymbol{x})}{p_s(\boldsymbol{x})}\frac1{1-\pi}\\
\propto& \frac{p_t(\boldsymbol{x})}{p_s(\boldsymbol{x})}\frac{1}{\kappa(\boldsymbol{x})}.
\end{aligned}
$$

\section{Additional Experiments}
\label{supp:simu}
In this section, we conducted additional experiments to evaluate the performance of our proposed method.

\subsection{Performances under Different Sample Sizes and Shifts}

In this part, we explore the performances of our proposed method under different sample sizes and shifts. 
We study on the three datasets: MNIST, CIFAR-10 and CIFAR-100.
The model setting are the same as the ones in Section 4.
As for the label shift mechanism, besides the Dirichlet shift, we also consider the “tweak-one” shift \cite{lipton2018detecting}, in which the the prior of the 4th class is $\rho$ and the other classes' priors are $(1-\rho)/(k-1)$ with total class number as $k$.
We consider the competitve methods from the followings:
\vspace{-0.15in}
\begin{itemize}
    \item BBSE-hard and BBSE-soft \cite{lipton2018detecting}
    \vspace{-0.1in}
    \item RLLS-hard and RLLS-soft \cite{azizzadenesheli2018regularized}
    \vspace{-0.1in}
    \item MLLS \cite{alexandari2020, garg2020}
\end{itemize}
\vspace{-0.15in}
We add the post-prediction calibrations to the adaptation methods except BBSE-hard and RLLS-hard, and the calibration methods include BCTS, NBVS, TS, VS.
Thus, we totally do the performance comparison over 22 methods.

In our comparison, we consider two metrics:
the first one is the weight MSE, i.e. $\|\hat{\bo} - \bo\|^2$; another is the prediction improvement, which is the delta accuracy of the domain-adapted model relative to the original model. 
The numerical results are summarized in Table~\ref{table.MINIST_Dirichlet}-\ref{table.CIFAR-100_tweakone_acc}.
For Dirichlet shift, we consider the shift parameter $\alpha\in \{0.1, 1.0\}$; and for tweakone shift, we allow the shift parameter $\rho \in \{0.01, 0.9\}$.
We range the sample size $n \in \{500,1500, 4500\}$.
Note that for the cases with CIFAR-100, we don't include the results with sample size $n = 500$.
Because when $n = 500$, some classes will have zero sample in the testing datasets.

First of all, it can be seen that, without calibration, our proposed method could always have the best estimation performance than the other methods and the improvements are very significant. 
For example, for the CIFAR-10 dataset, our method could reduce about $50\%$ MSE under Dirichlet shift and under tweakone shift with $\rho=0.9$.
By adding the post-prediction calibration, all the four methods (BBSE-soft, RLLS-soft, MLLS and ELSA) get improved, while MLLS and ELSA would usually be included in the top-3 results.
Note the estimation performances of our proposed method are very stable under different calibrations, but MLLS's performance significantly relies on the calibration method.
For example, for the MNIST dataset, MLLS could reach $3.556\times 10^{-3}$ MSE with BCTS while the MSE deteriorates to $2.199\times 10^{-2}$ when adopting TS. 
However, for our proposed method, the MSEs are always about $3.5\times 10^{-3}$ regardless the calibrations.
Thus, we believe our proposed method could consistently achieve the high accuracy.

\subsection{Computation Efficiency Study}

In this section, we conduct additional experiments to validate the computation efficiency of our method.
First of all, we show that even the proper calibration could improve the accuracy, the computation is costly. 
We record the computation time over different calibrations. 
We focus on the CIFAR-10 dataset and use the Dirichlet shift with $\alpha=0.1$.
We range the sample size from $2000$ to $10000$.
The results are shown in Table~\ref{tab.CIFAR-10_Cali_Time}.
It can be seen that the computation time increases almost linearly with the sample size.
VS is the most computational intensive method and then follows by BCTS.

\begin{table}[t]
\caption{The computation time (second) for different calibration methods on CIFAR-10. The Dirichlet shift parameter $\alpha$ is fixed as 0.1.}
\label{tab.CIFAR-10_Cali_Time}
\centering
\begin{tabular}{crrr}
\toprule
\multirow{2}{*}{Calibration} &  \multicolumn{3}{c}{Sample Size}  \\
  \cline{2-4}
 & 2000 &  5000 &  10000 \\
\midrule
TS          &     0.671 &     1.420 &      2.826 \\
VS          &     4.655 &    13.330 &     22.812 \\
NBVS        &     1.125 &     2.457 &      4.687 \\
BCTS        &     2.035 &     4.296 &      9.508 \\
\bottomrule
\end{tabular}
\end{table}

From our previous performance comparison study and the experiments in \cite{alexandari2020}, it has been observed that MLLS with BCTS would often reach the state-of-the-art performance.
Hence, we will mainly compare our method with MLLS with none and BCTS calibrations.
Note that as both MLLS and our ELSA are iterative methods, thus, in the following, we will examine the computation efficiency and estimation accuracy per iteration. 
In the experiment, we study on CIFAR-10 and CIFAR-100.
We adopt the Dirichlet shift and set sample size $n=10000$, shift parameter $\alpha=0.1$.
The other settings are the same as those in Section 4.
The results are shown in Figure~~\ref{fig.MSE_Iteration}.
It can be seen that our method (ELSA) could reach the lowest MSE even without the post-prediction calibration.
The convergence of ELSA is very fast and within 1-2 iterations.
In terms of the computation time, ELSA costs almost the same time as MLLS without calibration. 
Although the BCTS calibration could improve the estimation performance for MLLS, the computation for the calibration is costly. Thus, we think our proposed method is efficient in terms of computation and estimation accuracy.

\begin{figure}[t]
\centering
\begin{tabular}{@{}cc@{}}
\includegraphics[width=0.45\linewidth]{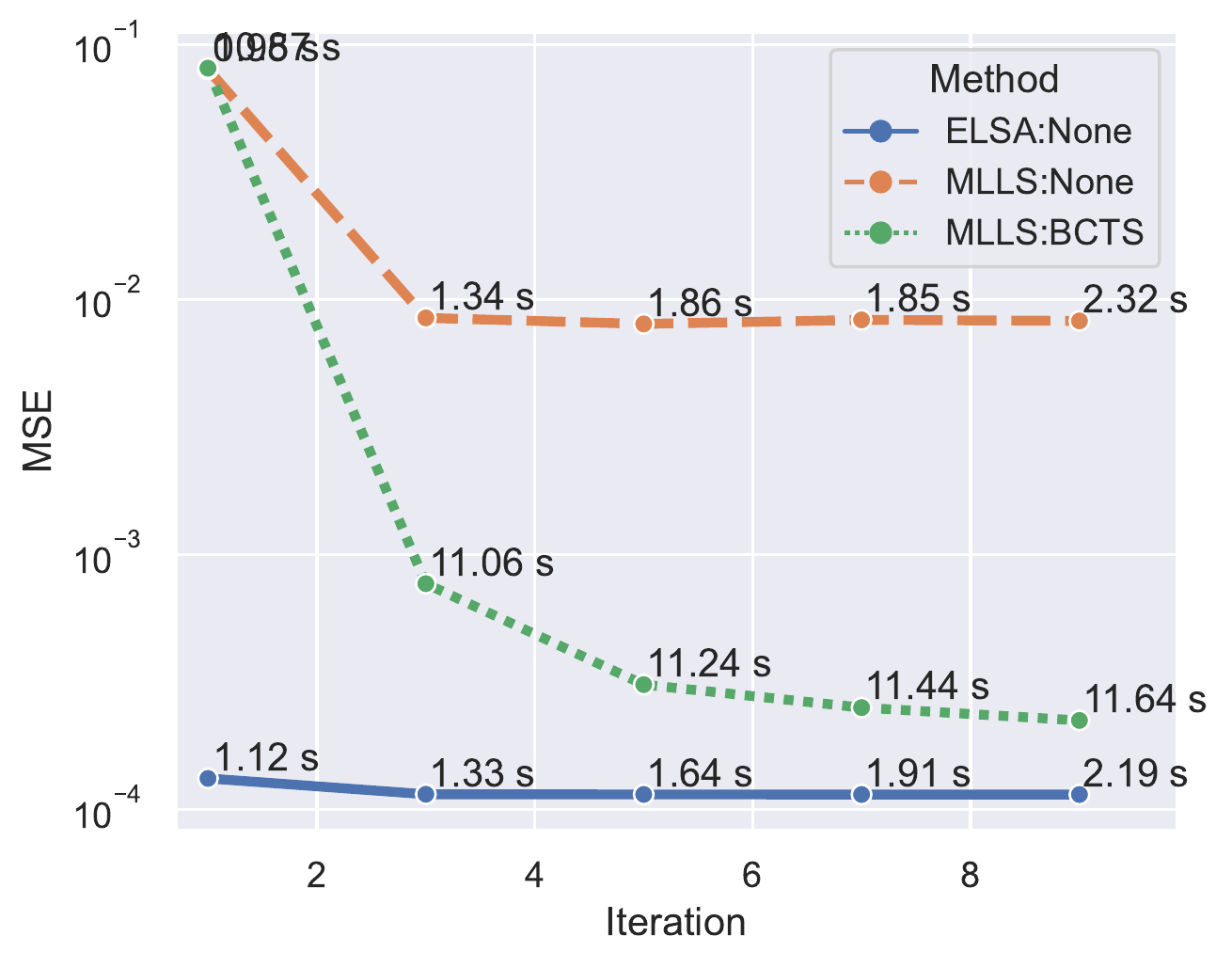}&
\includegraphics[width=0.45\linewidth]{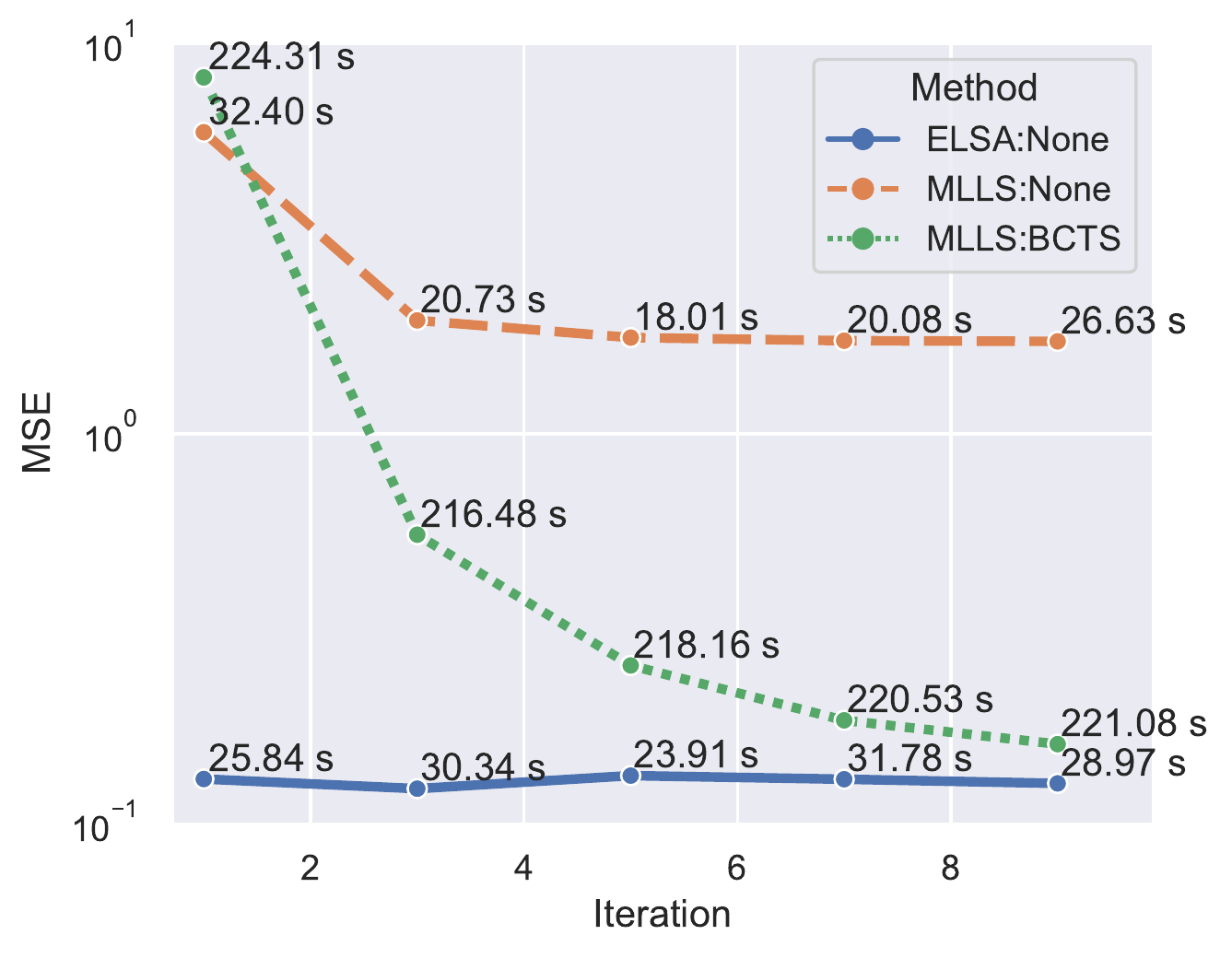}\\
~~~~(a) CIFAR-10 & ~~~~(b) CIFAR-100\\
\end{tabular}
    \caption{MSE v.s. Iterations. We adopt the Dirichlet shift and set sample size $n=10000$, shift parameter $\alpha=0.1$.}
    \label{fig.MSE_Iteration}
\end{figure}

\section{Proofs}\label{sup:proofs}

\subsection{Proof of Theorem~\ref{th:perpendicualr}}
We first state the full version of the theorem, adding the nuisance spaces into the theorem.
\begin{theorem*}
Without a loss of generality, we replace $w_k$ with $w_k=(1-\sum_{i=1}^{k-1}w_ip_i)/p_k$, and replace $p_k$ with $p_k=1-\sum_{i=1}^{k-1}p_i$, where $p_i\coloneqq \Pr{}_{\!s}(y=i)$.
The log-likelihood function is given by $r\log(\pi)+(1-r)\log(1-\pi)+\sum_{i=1}^{k-1}r\mathrm{I}(y=i)\left\{\log p(\x| y=i)+\log(p_i)\right\}+r\mathrm{I}(y=k)\left\{\log p(\x| y=k)+\log(1-p_1-\dots-p_{k-1})\right\}+(1-r)\log\{p(\x| y=1) w_1p_1+p(\x| y=2)w_2p_2+\dots+p(\x| y=k)(1-w_1p_1-\dots-w_{k-1}p_{k-1})\}$.
The nuisances are $p_1,\dots,p_{k-1}$, $p(\x| y=1)$,\dots, $p(\x| y=k)$, and $\pi$.
Then the nuisance tangent spaces are
\begin{enumerate}
    \item the nuisance tangent space with respect to $p(\x| y=j)$ is given by $
    \Lambda_j=\{r\mathrm{I}(y=j)\g_j(\x)+(1-r)p_t(y=j|\x)\g_j(\x):\g_j(\x)\in\mathbb{R}^{k-1},\int\g_j(\x)p(\x| y=j)d\x=\0\}$, for $j=1,\dots, k$.
    \item The nuisance tangent space with respect to $p_j$ is given by $\Lambda_{p_j}=\big([r\{{\mathrm{I}(y=j)}/{p_j}-{\mathrm{I}(y=k)}/{(1-p_1-\dots-p_{k-1})}\}+(1-r)\omega_j{\left\{p(\x| y=j)-p(\x| y=k)\right\}}/{p_t(\x)}]\b:\b\in\mathbb{R}^{k-1}\big)$, for $j=1,\dots,k-1$.
    \item The nuisance tangent space with respect to $\pi$ is given by $\Lambda_{\pi}=\big[\{{r}/{\pi}-{(1-r)}/{(1-\pi)}\}\a:\a\in\mathbb{R}^{k-1}\big]$.
\end{enumerate}
Then, the perpendicular space $\Lambda^{\perp}$, which is orthogonal to $\Lambda_j$ ($j=1,\dots,k$), $\Lambda_{p_j}$ ($j=1,\dots,k$), and $\Lambda_\pi$, is given by
\begin{align}
\Lambda^{\perp}=
\Bigg[
\frac{r}{\pi}\Bigg\{1-\sum_{i=1}^k \omega_i \mathrm{I}(y=i)\Bigg\} E(\mathbf{h} | y=k)\notag+\Bigg\{\frac{r}{\pi} \sum_{i=1}^k \omega_i \mathrm{I}(y=i)-\frac{1-r}{1-\pi}\Bigg\} \mathbf{h}(\mathbf{x}):E_t(\mathbf{h})=\0\in\mathbb{R}^{k-1}
\Bigg].
\end{align}
\end{theorem*}

\begin{proof}
The likelihood function is
\[
\begin{aligned}
	&\pi^r(1-\pi)^{1-r}\prod_{i=1}^{k-1}\left\{p(\x\mid y=i)p_i\right\}^{r\mathrm{I}(y=i)}\left\{p(\x\mid y=k)(1-p_1-\dots-p_{k-1})\right\}^{r\mathrm{I}(y=k)}\\
	&\left\{p(\x\mid y=1) w_1p_1+p(\x\mid y=2)w_2p_2+\dots+p(\x\mid y=k)(1-w_1p_1-\dots-w_{k-1}p_{k-1})\right\}^{1-r}.
\end{aligned}
\]
The log-likelihood is
\[
\begin{aligned}
	&r\log(\pi)+(1-r)\log(1-\pi)+\sum_{i=1}^{k-1}r\mathrm{I}(y=i)\left\{\log p(\x\mid y=i)+\log(p_i)\right\}+\\
	&r\mathrm{I}(y=k)\left\{\log p(\x\mid y=k)+\log(1-p_1-\dots-p_{k-1})\right\}+\\
	&(1-r)\log\left\{p(\x\mid y=1) w_1p_1+p(\x\mid y=2)w_2p_2+\dots+p(\x\mid y=k)(1-w_1p_1-\dots-w_{k-1}p_{k-1})\right\}
\end{aligned}
\]
The nuisance tangent space with respect to $p(\x\mid y=j)$ is
\[
\Lambda_j=\left\{r\mathrm{I}(y=j)\g_j(\x)+(1-r)p_t(y=j\mid\x)\g_j(\x):\g_j(\x)\in\mathcal{R}^{k-1},\int\g_j(\x)p(\x\mid y=j)d\x=\0\right\},
\]
for $j=1,\dots,k$.
The nuisance tangent space for $p_j=p_s(y=j)$ is
\[
\Lambda_{p_j}=\left[\left(r\left\{\frac{\mathrm{I}(y=j)}{p_j}-\frac{\mathrm{I}(y=k)}{1-p_1-\dots-p_{k-1}}\right\}+(1-r)\frac{p(\x\mid y=j)-p(\x\mid y=k)}{p_t(\x)}w_j\right)\b:\b\in\mathcal{R}^{k-1}\right],
\]
where $j=1,\dots,k-1$.
The nuisance tangent space with respect to $\pi$ is
\[
\Lambda_{\pi}=\left\{\left\{\frac{r}{\pi}-\frac{1-r}{1-\pi}\right\}\a:\a\in\mathcal{R}^{k-1}\right\}.
\]
The orthogonal space has a general form
\[
r\sum_{i=1}^{k}\mathrm{I}(y=i)\f_i(\x)+(1-r)\h(\x).
\]
Because it is orthogonal to $\Lambda_j$ for $j=1,\dots,k$
\[
\begin{aligned}
&E\left\{
r^2\mathrm{I}(y=j)\f_j\trans(\x)\g_j(\x)+(1-r)^2p_t(y=j\mid\x)\h\trans(\x)\g_j(\x)\right\}\\
=&\pi E_s\left\{\mathrm{I}(y=j)\f_j\trans\g_j\right\}+(1-\pi)E_t\left\{\mathrm{I}(y=j)\h\trans\g_j\right\}=0
\end{aligned}
\]
for $j=1,\dots,k$.
The previous equation implies that
\begin{equation}\label{eq:perp2}
\pi p_j\f_j(\x)+(1-\pi)p_jw_j\h(\x)=\pi p_jE\left\{\f_j(\x)\mid y=j\right\}+(1-\pi)p_jw_jE\left\{\h(\x)\mid y=j\right\}
\end{equation}
for $j=1,\dots,k$.
Similarly, this element is orthogonal to $\Lambda_{p_j}$
\begin{equation}\label{eq:perp3}
\begin{aligned}
&E\left\{\frac{r^2}{p_j}\mathrm{I}(y=j)\f_j(\x)-\frac{r^2}{1-p_1-\dots-p_{k-1}}\mathrm{I}(y=k)\f_k(\x)+(1-r)^2\frac{p(\x\mid y=j)-p(\x\mid y=k)}{p_t(\x)}w_j\h(\x)\right\}\\
=&\pi E(\f_j\mid y=j)-\pi E(\f_k\mid y=k)+(1-\pi)w_jE\left\{\h(\x)\mid y=j\right\}-(1-\pi)w_jE\left\{\h(\x)\mid y=k\right\}=\0,
\end{aligned}
\end{equation}
for $j=1,\dots,k-1$.
Equation~(\ref{eq:perp3}) implies that
\begin{equation}\label{eq:perp3-1}
\pi p_j E\left\{\f_j(\x)\mid y=j\right\}+(1-\pi)w_jp_jE\left\{\h(\x)\mid y=j\right\}=\pi p_jE\left\{\f_k(\x)\mid y=k\right\}+(1-\pi)w_jp_jE\left\{\h(\x)\mid y=k\right\}
\end{equation}
Equations~(\ref{eq:perp2}) and (\ref{eq:perp3-1}) imply that
\[
\pi p_j\f_j(\x)=\pi p_jE\left\{\f_k(\x)\mid y=k\right\}+(1-\pi)w_jp_jE\left\{\h(\x)\mid y=k\right\}-(1-\pi)p_jw_j\h(\x),
\]
and
\begin{equation}\label{eq:fj}
\f_j(\x)=E\left\{\f_k(\x)\mid y=k\right\}+\frac{1-\pi}{\pi}w_jE\left\{\h(\x)\mid y=k\right\}-\frac{1-\pi}{\pi}w_j\h(\x)
\end{equation}
and
\[
\begin{aligned}
r\sum_{j=1}^{k-1}\mathrm{I}(y=j)\f_j(\x)=&r\sum_{j=1}^{k-1}\mathrm{I}(y=j)E\left\{\f_k(\x)\mid y=k\right\}-r\sum_{j=1}^{k-1}\frac{1-\pi}{\pi}\mathrm{I}(y=j)w_j\h(\x)\\
&+r\sum_{i=1}^{k-1}\frac{1-\pi}{\pi}w_j\mathrm{I}(y=j)E\left\{\h(\x)\mid y=k\right\}.
\end{aligned}
\]
Also
\[
\f_k(\x)=E\left\{\f_k(\x)\mid y=k\right\}+\frac{1-\pi}{\pi}w_kE\left\{\h(\x)\mid y=k\right\}-\frac{1-\pi}{\pi}w_k\h(\x)
\]
and
\[
r\mathrm{I}(y=k)\f_k(\x)=r\mathrm{I}(y=k)E\left\{\f_k(\x)\mid y=k\right\}+r\frac{1-\pi}{\pi}w_k\mathrm{I}(y=k)E\left\{\h(\x)\mid y=k\right\}-r\frac{1-\pi}{\pi}w_k\mathrm{I}(y=k)\h(\x).
\]
So,
\begin{equation}\label{eq:firstpart-prtho}
\begin{aligned}
r\sum_{i=1}^{k}\mathrm{I}(y=i)\f_i(\x)&=r\sum_{i=1}^{k}\mathrm{I}(y=i)E\left\{\f_k(\x)\mid y=k\right\}+r\frac{1-\pi}{\pi}\sum_{i=1}^{k}w_i\mathrm{I}(y=i)E\left\{\h\mid y=k\right\}\\
&-r\frac{1-\pi}{\pi}\sum_{i=1}^{k}w_i\mathrm{I}(y=i)\h(\x)
\end{aligned}
\end{equation}
Next, we want to find $E\left\{\f_k(\x)\mid y=k\right\}$, taking conditional expectation $E(\cdot\mid y=j)$ on both sides of (\ref{eq:fj}) gives
\[
\pi E(\f_j\mid y=j)=\pi E(\f_k\mid y=k)-(1-\pi)w_jE\left\{\h(\x)\mid y=j\right\}+(1-\pi)w_jE\left\{\h(\x)\mid y=k\right\}
\]
thus
\[
\begin{aligned}
\pi\sum_{j=1}^{k-1}p_j E(\f_j\mid y=j)=&\pi\sum_{j=1}^{k-1}p_jE(\f_k\mid y=k)-(1-\pi)\sum_{j=1}^{k-1}w_jp_jE\left(\h\mid y=j\right)\\
	&+(1-\pi)\sum_{j=1}^{k-1}w_jp_jE\left\{\h(\x)\mid y=k\right\}
\end{aligned}
\]
Because the expectation is zero
\begin{equation}\label{eq:mean-zero}
	\begin{aligned}
		&\pi\sum_{i=1}^{k-1}p_iE\left(\f_i(\x)\mid y=i\right)+\pi(1-p_1-\dots-p_{k-1})E\left\{\f_k(\x)\mid y=k\right\}+(1-\pi)\sum_{i=1}^{k-1}w_ip_i E\left\{\h(\x)\mid y=i\right\}\\
		+&(1-\pi)\left(1-\sum_{i=1}^{k-1}w_ip_i\right)E\left\{\h(\x)\mid y=k\right\}=\0.
	\end{aligned}
\end{equation}
We have
\[
\pi E\left(\f_k\mid y=k\right)+(1-\pi)E\left(\h\mid y= k\right)=\0
\]
and
\begin{equation}\label{eq:fkk}
E(\f_k\mid y=k)=-\frac{1-\pi}{\pi}E\left(\h\mid y=k\right).
\end{equation}
By inserting (\ref{eq:fkk}) into (\ref{eq:firstpart-prtho}), the generic element in the perpendicular space is
\begin{align}
	&r\sum_{i=1}^{k}\mathrm{I}(y=i)\f_i(\x)+(1-r)\h(\x)\\
	=&r\sum_{i=1}^{k}\mathrm{I}(y=i)E\left(\f_k\mid y=k\right)+r\frac{1-\pi}{\pi}\sum_{i=1}^{k}w_i\mathrm{I}(y=i)E\left(\h\mid y=k\right)\\
	&-r\frac{1-\pi}{\pi}\sum_{i=1}^{k}w_i\mathrm{I}(y=i)\h(\x)+(1-r)\h(\x)\\
	=&rE\left(\f_k\mid y=k\right)+r\frac{1-\pi}{\pi}\sum_{i=1}^{k}w_i\mathrm{I}(y=i)E\left(\h\mid y=k\right)\\
	&-r\frac{1-\pi}{\pi}\sum_{i=1}^{k}w_i\mathrm{I}(y=i)\h(\x)+(1-r)\h(\x) \\
	=&-r\frac{1-\pi}{\pi}E(\h\mid y=k)+r\frac{1-\pi}{\pi}\sum_{i=1}^{k}w_i\mathrm{I}(y=i)E(\h\mid y=k)\allowdisplaybreaks[4]\\
	&-\left\{r\frac{1-\pi}{\pi}\sum_{i=1}^{k}w_i\mathrm{I}(y=i)-(1-r)\right\}\h(\x) \\
	\propto&\frac{r}{\pi}\left\{1-\sum_{i=1}^{k}w_i\mathrm{I}(y=i)\right\}E(\h\mid y=k)+\left\{\frac{r}{\pi}\sum_{i=1}^{k}w_i\mathrm{I}(y=i)-\frac{1-r}{1-\pi}\right\}\h(\x),
\end{align}
where $w_k=(1-p_1w_1-\dots-p_{k-1}w_{k-1})/p_{k}$ and $p_k=1-p_1-\dots-p_{k-1}$.
\end{proof}

\subsection{Proof of Theorem~\ref{th:asymp}}

\begin{theorem*}
Letting
\begin{align*}
    \mathbf{g}_p(\x,y,r)=\frac{r}{\pi}\left\{1-\sum_{i=1}^k \omega_i \mathrm{I}(y=i)\right\} E(\mathbf{h}_{\mathrm{ELSA}} | y=k)\notag\\
    +\left\{\frac{r}{\pi} \sum_{i=1}^k \omega_i \mathrm{I}(y=i)-\frac{1-r}{1-\pi}\right\} \mathbf{h}_{\mathrm{ELSA}}(\mathbf{x}),
\end{align*}
where $\mathbf{h}_{\mathrm{ELSA}}(\x)$ is defined in (\ref{eq:hp}).
Under the assumption that $E\bigg\{{\partial\mathbf{g}_p}/{{\partial\bo^{(-k)}}^{\mathrm{T}}}\bigg\}\big|_{\bo^{(-k)}=\bo^{(-k)}_0}$ is non-singular and $n/(n+m)\to\pi\in(0,1)$ as $n\to\infty$, the proposed estimator $\widehat{\bo}^{(-k)}$, as the solution to (\ref{eq:empirical-equation}),
\begin{enumerate}
\item is $\sqrt{n}$-consistent: $\widehat{\bo}^{(-k)}={\bo}^{(-k)}_0+o_p({1}/{\sqrt{n}})$.
That is, for any small $\epsilon>0$, we can find a value $N$ such that for all $n>N$, $\Pr(\sqrt{n}\Vert\widehat{\bo}^{(-k)}-\bo_0^{(-k)}\Vert>\epsilon)<\epsilon$.
\item has an asymptotic normal limit as $n\to\infty$, we have
$\sqrt{n}\left(\widehat{\bo}^{(-k)}-\bo_0^{(-k)}\right)\xrightarrow{d}
\mathcal{N}\left(\mathbf{0}, \pi\mathbf{U}\mathbf{V}\mathbf{U}^\top\right)$.
\end{enumerate}
where $\mathbf{U}=\bigg[E\bigg\{{\partial\mathbf{g}_p}/{{\partial\bo^{(-k)}}^{\mathrm{T}}}\bigg\}\big|_{\bo^{(-k)}=\bo^{(-k)}_0}\bigg]^{-1}$ and $\mathbf{V} =E\{\mathbf{g}_p(\x,y,r;\bo_0^{(-k)})\mathbf{g}_p^{\mathrm{T}}(\x,y,r;\bo_0^{(-k)})\}$.
\end{theorem*}

\begin{proof}
Recall that the proposed estimator is derived from
\[
\0=\sum_{i=1}^{n+m}\mathbf{g}_p(\x_i,y_i,r_i;\widehat{\bo}^{(-k)}).
\]
We can do Tyler expansion at $\bo_0$, then we have
\[
\0=\sum_{i=}^{n+m}\mathbf{g}_p(\x_i,y_i,r_i;\bo^{(-k)}_0)+\left\{\sum_{i=1}^{n+m}\frac{\partial\mathbf{g}_p(\bo^{(-k)\ast})}{\partial\bo^{(-k)\mathrm{T}}}\right\}\left(\widehat{\bo}^{(-k)}-\bo_0^{(-k)}\right).
\]
With mild regularity assumptions, we have
\[
\frac{1}{n+m}\left\{\sum_{i=1}^{n+m}\frac{\partial\mathbf{g}_p(\bo^{(-k)\ast})}{\partial\bo^{(-k)\mathrm{T}}}\right\}\xrightarrow{p}E\left\{\frac{\partial\mathbf{g}_p(\x,y,r;\bo_0)}{\partial\bo^{(-k)\mathrm{T}}}\right\},
\]
and by the nonsingularity assumption
\[
\left\{\frac{1}{n+m}\sum_{i=1}^{n+m}\frac{\partial\mathbf{g}_p(\bo^{(-k)\ast})}{\partial\bo^{(-k)\mathrm{T}}}\right\}^{-1}\xrightarrow{p}\left[E\left\{\frac{\partial\mathbf{g}_p(\x,y,r;\bo_0)}{\partial\bo^{(-k)\mathrm{T}}}\right\}\right]^{-1}.
\]
Therefore,
\[
\begin{aligned}
\sqrt{n+m}\left(\widehat{\bo}^{(-k)}-\bo_0^{(-k)}\right)=&-\left\{\frac{1}{n+m}\sum_{i=1}^{n+m}\frac{\partial\mathbf{g}_p(\bo^{(-k)\ast})}{\partial\bo^{(-k)\mathrm{T}}}\right\}^{-1}\left\{\frac{1}{\sqrt{n+m}}\sum_{i=1}^{n+m}\mathbf{g}_p(\x_i,y_i,r_i;\bo^{(-k)}_0)\right\}\\
=&-\left[E\left\{\frac{\partial\mathbf{g}_p(\x,y,r;\bo_0)}{\partial\bo^{(-k)\mathrm{T}}}\right\}\right]^{-1}\left\{\frac{1}{\sqrt{n+m}}\sum_{i=1}^{n+m}\mathbf{g}_p(\x_i,y_i,r_i;\bo^{(-k)}_0)\right\}+o_p(1)
\end{aligned}
\]
Because $\mathbf{g}_p(\x,y,r;\bo^{(-k)})$ is an element in a mean-zero function space, we have
\[
E\left\{\mathbf{g}_p(\x,y,r;\bo^{(-k)})\right\}=\0,
\]
we immediately deduce that the influence function of $\widehat{\bo}^{(-k)}$ is given by
\[
\varphi(\x,y,r)=-\left[E\left\{\frac{\partial\mathbf{g}_p(\x,y,r;\bo_0)}{\partial\bo^{(-k)\mathrm{T}}}\right\}\right]^{-1}\mathbf{g}_p(\x_i,y_i,r_i;\bo^{(-k)}_0).
\]
By the property of the RAL estimator (as described in Section~\ref{sup:pre}), we have
\[
\sqrt{n+m}\left(\widehat{\bo}^{(-k)}-{\bo}_0^{(-k)}\right)\xrightarrow{d}\mathcal{N}(\0,\mathbf{U}\mathbf{V}\mathbf{U}^\top),
\]
where
$$
\mathbf{U}=\bigg[E\bigg\{{\partial\mathbf{g}_p}/{{\partial\bo^{(-k)}}^{\mathrm{T}}}\bigg\}\big|_{\bo^{(-k)}=\bo^{(-k)}_0}\bigg]^{-1}
$$
and
$$
\mathbf{V} =E\{\mathbf{g}_p(\x,y,r;\bo_0^{(-k)})\mathbf{g}_p^{\mathrm{T}}(\x,y,r;\bo_0^{(-k)})\}.
$$
By the assumption that $n/(n+m)\to\pi$ as $n\to\infty$, we have
\[
\sqrt{n}\left(\widehat{\bo}^{(-k)}-\bo_0^{(-k)}\right)=\frac{\sqrt{n}}{\sqrt{n+m}}\sqrt{n+m}\left(\widehat{\bo}^{(-k)}-\bo_0^{(-k)}\right)\xrightarrow{d}\mathcal{N}\left(\0,\pi\mathbf{U}\mathbf{V}\mathbf{U}^\top\right).
\]
Thus, we have finished the second part of the theorem.
Because we have prove the asymptotic normality of the proposed estimator, immediately we have
\[
\widehat{\bo}^{(-k)}-\bo_0^{(-k)}=o_p\left(\frac{1}{\sqrt{n+m}}\right)=o_p\left(\frac{1}{\sqrt{n}}\right).
\]

\end{proof}


\end{document}